\documentclass[twoside]{article}

\usepackage[accepted]{aistats4arxiv}
\usepackage{amsmath,amsfonts,bm}

\def\eqref#1{equation~\ref{#1}}

\def\1{\bm{1}}

\DeclareMathAlphabet{\mathsfit}{\encodingdefault}{\sfdefault}{m}{sl}
\SetMathAlphabet{\mathsfit}{bold}{\encodingdefault}{\sfdefault}{bx}{n}

\newcommand{\R}{\mathbb{R}}

\usepackage[T1]{fontenc}
\usepackage[hidelinks]{hyperref}
\usepackage{url}

\usepackage{graphicx} 
\usepackage{amsfonts} 
\usepackage{amsthm}   
\usepackage{nicefrac}       
\usepackage{microtype}      
\usepackage{mdframed}

\usepackage{xcolor}         
\usepackage{tikz,lipsum,lmodern}
\usepackage{bm}
\usepackage{booktabs}       
\usepackage{wrapfig}
\usepackage{colortbl}
\usepackage{relsize}
\usepackage{thm-restate}
\usepackage{natbib}
\usepackage{mathtools}
\usepackage{sidecap}

\usepackage[most]{tcolorbox}
\usepackage{pifont}

\usepackage[capitalize]{cleveref}
\creflabelformat{equation}{#2\textup{#1}#3}
\renewcommand{\eqref}[1]{(\ref{#1})}  

\theoremstyle{plain}
\newtheorem{theorem}{Theorem}[section]

\newtheorem{lemma}[theorem]{Lemma}

\theoremstyle{definition}

\theoremstyle{remark}

\newenvironment{smallpmatrix}
  {\left(\begin{smallmatrix}}
  {\end{smallmatrix}\right)}
  
\renewcommand{\paragraph}[1]{\textbf{#1}~~} 
\newcommand{\cmark}{\ding{51}}%
\newcommand{\xmark}{\ding{55}}%

\newcommand{\J}[0]{\mathbf{J}}      
\newcommand{\M}[0]{\mathbf{M}}      
\newcommand{\U}[0]{\mathbf{U}}      
\newcommand{\mat}[1]{\mathbf{#1}}   
\renewcommand{\vec}[1]{\mathbf{#1}} 
\newcommand{\btheta}{\boldsymbol{\theta}} 
\newcommand{\bmu}{\boldsymbol{\mu}} 
\newcommand{\bSigma}{\boldsymbol{\Sigma}} 
\newcommand{\x}[0]{\vec{x}}         
\newcommand{\y}[0]{\vec{y}}         
\newcommand{\N}[0]{\mathcal{N}}     
\newcommand{\T}[0]{^{\top}}        
\newcommand\norm[1]{\left\lVert#1\right\rVert}
\newcommand{\MAP}[0]{\textnormal{\textsc{map}}}
\newcommand{\trace}[0]{\mathrm{Tr}}

\newcommand{\first}[1]{\textbf{#1}}

\definecolor{myblue}{RGB}{218, 232, 252}

\begin{document}

\twocolumn[
\aistatstitle{Bayes without Underfitting: Fully Correlated Deep Learning Posteriors via Alternating Projections}

\aistatsauthor{ Marco Miani$^{\dagger}$ \And Hrittik Roy$^{\dagger}$ \And Søren Hauberg }

\aistatsaddress{ Technical University of Denmark \\ \texttt{mmia@dtu.dk} \And  Technical University of Denmark \\ \texttt{hroy@dtu.dk} \And Technical University of Denmark \\ \texttt{sohau@dtu.dk}} ]


\begin{abstract}
  Bayesian deep learning all too often underfits so that the Bayesian prediction is less accurate than a simple point estimate. Uncertainty quantification then comes at the cost of accuracy. For linearized models, the null space of the generalized Gauss-Newton matrix corresponds to parameters that preserve the training predictions of the point estimate. We propose to build Bayesian approximations in this null space, thereby guaranteeing that the Bayesian predictive does not underfit. We suggest a matrix-free algorithm for projecting onto this null space, which scales linearly with the number of parameters and quadratically with the number of output dimensions. We further propose an approximation that only scales linearly with parameters to make the method applicable to generative models. An extensive empirical evaluation shows that the approach scales to large models, including vision transformers with 28 million parameters. Code is available at: {\texttt{https://github.com/h-roy/projected-bayes}}
\end{abstract}


\section{Underfitting in Bayesian deep learning}
Bayesian deep learning tends to underfit. Numerous studies demonstrate that marginalizing approximate weight posteriors yields less accurate predictions than applying a \emph{maximum a posteriori} (\textsc{map}) point estimate \citep{wenzel2020good, daxberger2021laplace, zhang2024the, kristiadi2022being}. This significantly reduces the practical value of Bayesian deep learning, despite having attractive theoretical properties \citep{germain2016pac}.
To counteract this trend, we first explicate \emph{why} underfitting happens with Gaussian approximate posteriors, and thereafter propose a solution for low-noise data (e.g.\@ images).

\begin{figure}
  \includegraphics[width=\linewidth]{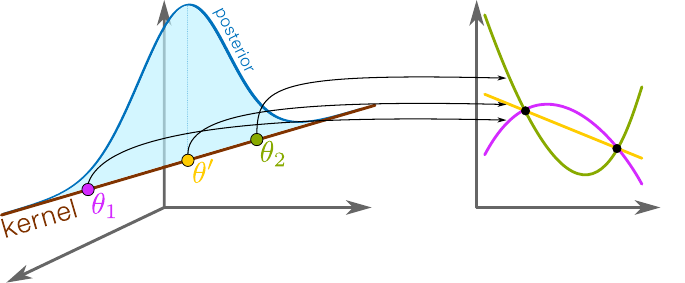}
  \vspace{-7mm}
  \caption{\emph{Key idea:} In overparametrized linear models, the \emph{kernel} (null space) contains all models that have identical predictions on the training data. We propose restricting approximate posteriors of deep neural networks to this kernel to avoid underfitting.}
  \label{fig:idea}
\end{figure}
\paragraph{What is underfitting?}
Deep learning performs very well when the training data is subject to limited observation noise. In contrast, Bayesian deep learning often \emph{underfits} in the sense that the Bayesian prediction deviates significantly from a \emph{point estimate} prediction \emph{on the training data} $\mathcal{D}$, i.e.
\begin{align}
  \mathbb{E}_{\btheta \sim q} \left[f(\btheta, \x)\right] \neq f(\btheta_{\textsc{map}}, \x)
  \qquad
  \text{for } \x \in \mathcal{D},
\end{align}
where $q$ is an approximate posterior and $\btheta_{\textsc{map}}$ is a point estimate of the neural network $f$.
This notion of underfitting only applies to low-noise data, where the posterior should have little uncertainty on the training data, but this is the most prominent scenario in deep learning. \looseness=-1 

\textbf{Why do Bayesian neural networks underfit?}
Deep neural networks are commonly \emph{overparametrized}, i.e.\@ the model has more parameters than observations, as this tends to improve both optimization and generalization \citep{allen2019learning}. Fig.~\ref{fig:idea} sketches the situation which we later formalize: Consider fitting a quadratic function with three parameters to only two observations. This leaves one degree of freedom undetermined and a linear subspace of the model parameters exists wherein all parameters yield identical predictions on the training data. For low-noise data, the true posterior will concentrate around this subspace. Unfortunately, common Gaussian approximations to the posterior are axis-aligned (i.e.\@ \emph{mean field} approximations \citep{bishop2006pattern}) and not aligned with the mentioned subspace. This mismatch leads to a poor posterior approximation that tends to underfit. \emph{We suggest, quite simply, to project the approximate posterior to the subspace in which underfitting cannot happen.}

\paragraph{Why is this beneficial?}
Our proposed posterior reflects the degrees of freedom in the model that fundamentally cannot be determined by even noise-free data. Predicting according to this distribution gives reliable out-of-distribution detection and general uncertainty quantification without underfitting. 
Our approach captures correlations between layers and retains high-accuracy predictions while scaling linearly with model size. 
Table~\ref{tab:models} contrasts our method with existing ones.\looseness=-1

\paragraph{Why is this difficult?}
We will soon see that the relevant subspace on which to project is given by the \emph{kernel} (i.e.\@ \emph{null space}) of a matrix that is quadratic in the number of model parameters. Even for models of modest size, this matrix is too large to be stored in memory and direct projection methods cannot be applied. \emph{We propose a linear-time sampling algorithm that can be implemented entirely using automatic differentiation.}


\begin{table*}[t]
	\vspace*{-1.0\baselineskip}
	\centering
        \begin{center}
	\caption{Comparisons between Laplace approximations. $N$: number of data points, $O$: output dimensions, $P$: number of parameters, $P_{l,in}$, $P_{l,out}$: input and output dimensions of layer $l$, $P_{ll}$: last layer parameters.} 
 \vspace{-3mm}
    \label{tab:models}
	\resizebox{1.0\textwidth}{!}{%
		\begin{tabular}{cccccccc}
			\toprule
			\textsc{Method} & \textsc{Correlation} & \textsc{Space}      & \textsc{per-sample Time}       & \textsc{preprocessing Time}      & \textsc{Error} & \textsc{Model}    & \textsc{Tractable} \\
                            & \textsc{structure}   & \textsc{complexity} & \textsc{complexity} & \textsc{complexity} & \textsc{bound} & \textsc{agnostic} & optimal \textsc{lml} \\
			\midrule
			Diagonal & \includegraphics[width=0.03\linewidth]{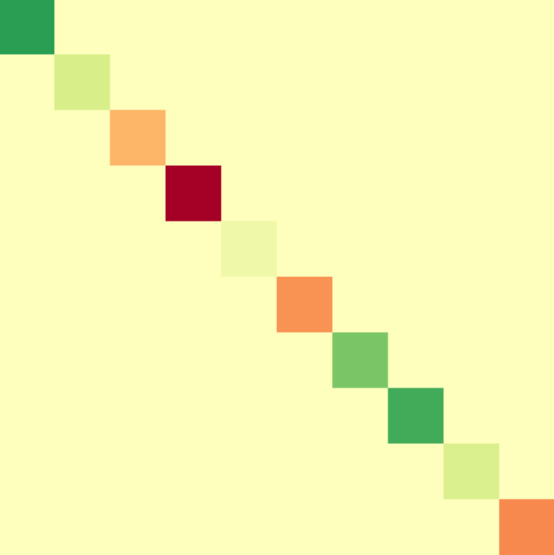} & $P$ & P & $PNO$ & \xmark & \cmark & \xmark \\
			Kronecker-Factored & \includegraphics[width=0.03\linewidth]{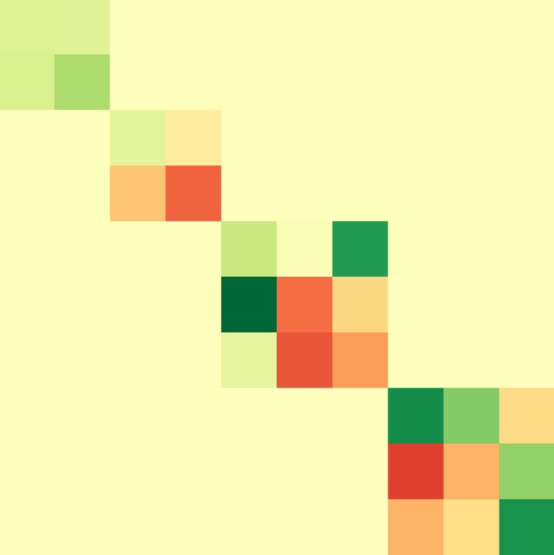} & $\sum_{l=1}^L P_{l,in}^2 + P_{l,out}^2$  & $\sum_{l=1}^L P_{l,in}^2 \!+\! P_{l,out}^2$ & $\sum_{l=1}^L \! NP_{l,in} \!+\! NP_{l,out} \!+\! P_{l,in}^3 \!+\! P_{l,out}^3$ 
   & \xmark & \xmark & \xmark\\
   			Last-Layer & \includegraphics[width=0.03\linewidth]{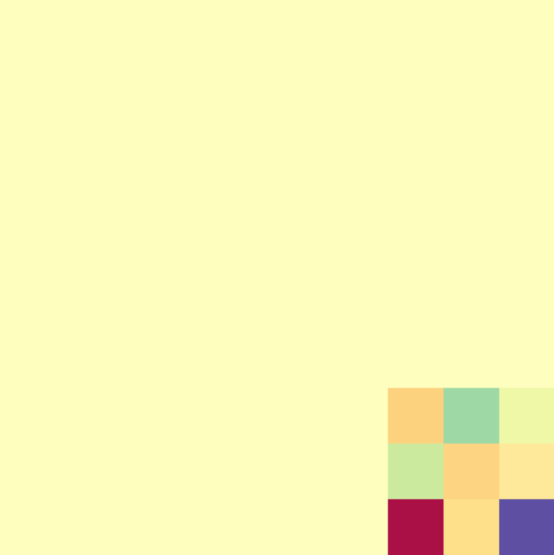} & $P_{ll}^2$ & $P_{ll}^2$ & $O^2N P_{ll}^2 + P_{ll}^3$ & \xmark & \cmark & \xmark\\
			\textbf{This paper} 
            & \includegraphics[width=0.03\linewidth]{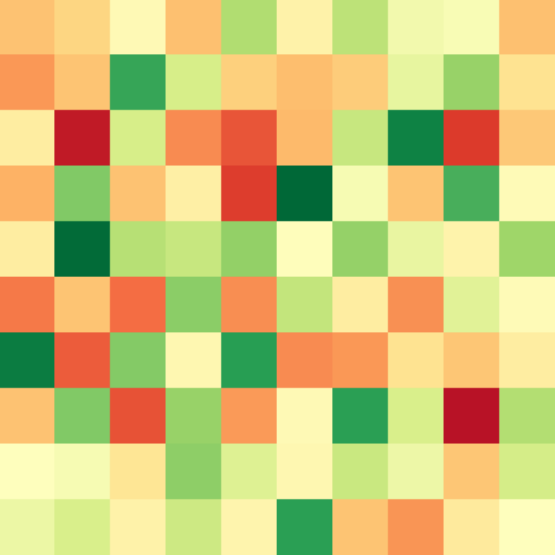} & $P+NO^2$ or $P+N$ & $t_{\max}PN+NO^3$ or $t_{\max}PN$ & 0 & \cmark~(Lemmas~\ref{lm:error_bound} \& \ref{lm:loss_proj_err})  & \cmark & \cmark~(Sec.~\ref{sec:3.3}) \\
			\bottomrule
	\end{tabular}}
        \end{center}	
	\vspace*{-0.5\baselineskip}
\end{table*}

\paragraph{Paper outline.}
Sec.~\ref{sec:background} gives the background to derive our approach. A wider discussion of related work is postponed to Sec.~\ref{sec:related_work}. We develop our approach in two steps. First, Sec.~\ref{sec:model} describes our proposed posterior approximation, while Sec.~\ref{sec:algorithm} derives an efficient sampling algorithm. Empirical investigations are conducted in Sec.~\ref{sec:results}.\looseness=-1

\section{Background and notation}\label{sec:background}

\paragraph{Notation.}
Let $f:\R^P \times \R^I \rightarrow \R^O$ denote a neural network with parameter $\btheta \in \R^P$ that maps inputs $\x \in \R^I$ to outputs $\y \in \R^O$.
We define the per-datum Jacobian as $\J_{\btheta}(\x) = \partial_{\btheta} f(\btheta, \x) \in \R^{O \times P}$ and its stacked counterpart $\J_{\btheta} = [\J_{\btheta}(\x_1); \ldots; \J_{\btheta}(\x_N)] \in \R^{NO \times P}$ for a fixed training dataset $\mathcal{D} = \{\x_1,\ldots,\x_N\}$.

\textbf{Gaussian approximate posteriors}, i.e.\@ $p(\btheta | \mathcal{D}) \approx q(\btheta | \mathcal{D}) = \N(\btheta | \bmu, \bSigma)$, are ever-present in Bayesian deep learning \citep{blundell2015weight, botev2017practical, sharma2021sketching, khan2018fast, maddox2019simple, stephan2017stochastic, osawa2019practical, antoran2022sampling}. For modern neural networks, the parameter dimension $P$ can easily be several billions, such that even storing the covariance $\bSigma$ in memory is infeasible. To allow for linear scaling, common Gaussian approximate posteriors are realized by disregarding most correlations in $\bSigma$, e.g.\@ with diagonal or block-diagonal approximations.

Unfortunately, there is significant theoretical evidence against disregarding correlations. \citet{foong2019between} shows that diagonal approximations underestimate `in-between' uncertainty, while \citet{reparam:2024} shows that uncorrelated models cannot generally be infinitesimally invariant to reparametrizations even if this property is held by the true posterior.\looseness=-1

\textbf{The \emph{linearized Laplace approximation}} (\emph{\textsc{lla}};  \citet{immer2021improving, khan2019approximate}) use a first-order Taylor expansion of $f(\btheta, \x) \approx f_{\text{lin}}^{\btheta_{\textsc{map}}}(\btheta, \x) = f(\btheta_{\textsc{map}}, \x) + \J_{\btheta_{\textsc{map}}}(\x)(\btheta - \btheta_{\textsc{map}})$, where $\btheta_{\textsc{map}}$ is the expansion point (usually a \textsc{map} estimate of the parameter). Secondly, \textsc{lla} perform a standard Laplace approximation \citep{mackay1992laplace} of the linearized model, yielding
\begin{align}
\label{eq:laplace_posterior_original}
  q_{\textsc{lla}}(\btheta | \mathcal{D}) &= \N\left(\btheta ~\big|~ \btheta_{\textsc{map}},\, \left(\alpha\mathbb{I} + \textsc{ggn}_{\btheta_{\textsc{map}}} \right)^{-1} \right) \\
  \textsc{ggn}_{\btheta_{\textsc{map}}} &= \J_{\btheta_{\textsc{map}}}\T \mat{H}_{\btheta_{\textsc{map}}} \J_{\btheta_{\textsc{map}}} \in \R^{P \times P}.
\end{align}
Here $\textsc{ggn}_{\btheta_{\textsc{map}}}$ is the so-called \emph{generalized Gauss-Newton matrix}, $\alpha$ is the prior precision, $\mat{H}_{\btheta}(\x) = -\partial^2_{f(\btheta, \x)} \log p(\y | f(\btheta, \x)) \in \R^{O \times O}$ and $\mat{H}_{\btheta} \in \R^{NO\times NO}$ is its stacked counterpart. This Hessian takes a simple closed-form for common likelihoods, e.g.\@ it is the identity for Gaussian likelihoods \citep{immer2021improving}. For notational simplicity, we treat this Hessian as an identity for the remainder of this paper.

The \textsc{lla} has a strong empirical performance \citep{immer2021improving}, but it is computationally taxing as the \textsc{ggn} has size $P \times P$. This is infeasible to instantiate even for models of modest size, and the $\mathcal{O}(P^3)$ cost associated with sampling presents a significant computational challenge. Iterative matrix-free solvers can potentially alleviate these concerns \citep{reparam:2024}, but the \textsc{ggn} is highly ill-conditioned \citep{papyan2020traces,miani:sketching:2024} and this approach requires overcoming issues of numerical stability. Practical \textsc{lla} implementations therefore resort to sparse, e.g.\@ diagonal or block-diagonal, approximations of the \textsc{ggn}.

\section{The proposed approximate posterior}\label{sec:model}
We next propose a fully correlated Gaussian posterior that is guaranteed to not underfit. Unless otherwise stated, the presented results are novel contributions and proofs of all theorems can be found in the appendix.


As alluded to, we propose restricting the posterior covariance to a particular subspace of the parameter space. Letting $\U \in \R^{P \times R}$ denote an orthogonal basis of this subspace, we consider an \emph{isotropic} model therein,
\begin{equation}
    q_{\textnormal{\textsc{proj}}}(\btheta | \btheta_{\textsc{map}}, \mathcal{D})
    =
    \mathcal{N}\left(\btheta_{\textsc{map}}, \alpha^{-1}\U\U\T \right).
    \label{eq:approx_post}
\end{equation} 
Specifically, we choose the above subspace as the \emph{kernel} (i.e.\@ \emph{null space}) of the \textsc{ggn} matrix and call the result the \emph{projected posterior}. Formally,
\begin{align}
  \U\U\T = \mathbb{I}_P - \mathcal{P}(\textsc{ggn}) = \mathbb{I}_P - \mat{V} \ [\mat{D} \!>\! 0] \ \mat{V}\T,
\end{align}
where $\mat{V}\mat{D}\mat{V}\T$ is an eigen decomposition of the \textsc{ggn} matrix and $[\mat{D} \!>\! 0] \in \R^{P \times P}$ is diagonal with elements 1 where $\mat{D}$ holds positive values. We use the shorthand notation $\mathcal{P}(\textsc{ggn})$ to denote the projection onto the image of the \textsc{ggn}.
\emph{We next justify the projected posterior through the lens of underfitting.}

%
\paragraph{The projected posterior never underfits.}
%
%
When using the linearized neural network, $f_{\text{lin}}^{\btheta_{\textsc{map}}}(\btheta, \x) = f(\btheta_{\textsc{map}}, \x) + \J_{\btheta_{\textsc{map}}}(\x)(\btheta - \btheta_{\textsc{map}})$, we avoid underfitting when $\J_{\btheta_{\textsc{map}}}(\x)(\btheta - \btheta_{\textsc{map}}) = \vec{0}$ on the training data. The linearized model then avoids underfitting when $\btheta - \btheta_{\textsc{map}}$ is restricted to the kernel (null space) of $\J_{\btheta_{\textsc{map}}} = [\J_{\btheta_{\textsc{map}}}(\x_1); \ldots; \J_{\btheta_{\textsc{map}}}(\x_N)]$. For the proposed projected posterior \eqref{eq:approx_post}, we, thus, choose $\U$ as an orthonormal basis of the Jacobian kernel and note that this kernel coincides with the zero eigenvectors of the \textsc{ggn}. \looseness=-1

These considerations can be formalized as follows.
\begin{lemma}
\label{lemma:3.1}
    The projected posterior \eqref{eq:approx_post} is supported on equal functions on the training data, i.e.\@ $\forall \x \in \mathcal{D}$
    \begin{align}
        &f_{\textnormal{lin}}^{\btheta_{\MAP}}(\btheta, \x) = f(\btheta_{\MAP}, \x)
        &\forall \btheta \sim q_{\textnormal{\textsc{proj}}},
    \end{align}
    which implies that~~$\textnormal{\textsc{Var}}_{\btheta \sim q_{\textnormal{\textsc{proj}}}} f_{\textnormal{lin}}^{\btheta_{\MAP}}(\btheta, \x) = 0$.
\end{lemma}
We emphasize that the statement \emph{only holds on the training data}. Elsewhere, the approximate posterior yields strictly positive variances with high probability under strong but reasonable assumptions.
\begin{restatable}{lemma}{PositiveVarianceLemma}
  \label{lm:variance_nonzero_nontrainset}
    Let \mbox{$\J_{\btheta} \!=\! [\J_{\btheta}(\x_1)\T \ldots \J_{\btheta}(\x_N)\T]\T$ and} $\x_{\textnormal{test}} \in \R^I$, then
    \begin{equation}
    \label{eq:variance_nonzero_nontrainset}
        \textnormal{\textsc{Var}}_{\btheta \sim q_{\textnormal{\textsc{proj}}}} f_{\textnormal{lin}}^{\btheta_{\MAP}}(\btheta, \x_{\textnormal{test}}) > 0
    \end{equation}
    if
    \resizebox{0.95\linewidth}{!}{$\textnormal{\textsc{Rank}}\!\left(\begin{smallpmatrix}
        \J_{\btheta_{\textsc{map}}} \\ \J_{\btheta_{\textsc{map}}}(\x_{\textnormal{test}})
    \end{smallpmatrix} \begin{smallpmatrix}
        \J_{\btheta_{\textsc{map}}} \\  \J_{\btheta_{\textsc{map}}}(\x_{\textnormal{test}})
    \end{smallpmatrix}\T\right)
    >
    \textnormal{\textsc{Rank}}\left(\J_{\btheta_{\textsc{map}}} \J_{\btheta_{\textsc{map}}}\T\right)$.} 
\end{restatable}
Jointly the two lemmas show that out-of-distribution data is guaranteed to have higher predictive variance than the training data. The (technical) rank assumption in Lemma~\ref{lm:variance_nonzero_nontrainset} is known to be satisfied with high probability under reasonable assumptions \citep{bombari2022memorization, nguyen2021tight, karhadkar2024bounds}. \looseness=-1

These results can be contrasted with \textsc{lla}, where we can show the following result.
%
\begin{restatable}{theorem}{UpperBoundLinearizedLaplace}
\label{thm:upper_bound_lla}
    For $\alpha>0$, the predictive variance of \textsc{lla} on any training datapoint is positive and bounded, 
    \begin{equation*}
        \frac{O \gamma^2}{\gamma^2 + \alpha}
        \leq
        \textnormal{\textsc{Var}}_{\theta\sim q_{\textnormal{\textsc{lla}}}} f_{\textnormal{lin}}^{\btheta_{\MAP}}(\btheta, \x) 
        \leq 
        \frac{O \lambda^2}{\lambda^2 + \alpha} 
        \quad \text{for } \x \in \mathcal{D}.
    \end{equation*}
    Here, $\lambda$ and $\gamma$ are the largest and smallest singular values of the dataset Jacobian $\J_{\btheta_{\textsc{map}}}$, respectively.
\end{restatable}
Here, $\gamma$ is strictly positive with high probability under reasonable assumption \citep{bombari2022memorization, nguyen2021tight, karhadkar2024bounds}. \emph{\textsc{lla} then underfits with high probability, which contrasts our approach (c.f.\@ Lemma~\ref{lemma:3.1})}.
Note that this is only true for the (intractable) \textsc{ggn}-based covariance. We expect that sparse approximations to this covariance, e.g.\@ diagonal and \textsc{kfac}, has even higher training set variance.



\paragraph{Underfitting in existing models.}
The above analysis justifies the projected posterior and also sheds light on why current Bayesian approximations often underfit. For efficiency, mean field approximations of the posterior covariance are quite common, e.g.\@ \emph{diagonal} or \emph{Kronecker factored} covariances \citep{ritter2018scalable, martens2015optimizing}. These approximations will generally sample outside the \textsc{ggn} kernel, implying strictly positive predictive variances on the training data even for noise-free data. This reasoning also motivates recent works that have attempted to improve these approximations by making their spectrum more aligned with the spectrum of the \textsc{ggn} \citep{george2018fast, dhahri2024shaving}. 

 

\paragraph{Computational benefits.}
Beyond the above theoretical motivations, our projected covariance also brings computational benefits. Since $\U$ is an orthonormal basis, the covariance $\U\U\T$ is a projection matrix, implying that its eigenvalues are all 0 or 1. This, in turn, implies that $\U\U\T = (\U\U\T)^2 = (\U\U\T)\T = (\U\U\T)^{\dagger}$, where $^{\dagger}$ denotes pseudo-inversion. Drawing samples from $q$, thus, only requires matrix-vector products with $\U\U\T$ and the matrix need not be instantiated.
The only open question is then how to efficiently perform such matrix-vector products. We answer this in Sec.~\ref{sec:algorithm}.

Since all eigenvalues of $\U\U\T$ are 0 or 1, we can a priori expect the matrix to yield numerically stable computations even at moderate numerical precision. This is in contrast to the \textsc{ggn}, which easily has condition numbers that exceed $10^6$ \citep{miani:sketching:2024}.

\paragraph{Tractable model selection.}
\label{sec:3.3}
A benefit of Bayesian neural networks is that we can choose $\alpha$ by maximizing the marginal likelihood, $p(\mathcal{D}|\alpha)$, on the training data, i.e.\@ without relying on validation data \citep{mackay1995probable}. However, since the marginal likelihood of the true posterior is intractable, it is common to consider that of the Laplace approximation \citep{immer2021scalable} 
\begin{align}
\label{eq: lml}
\begin{split}
    \log \, q_{\textsc{lla}}(\mathcal{D}|\alpha) 
      &= 
      \log\,p(\mathcal{D}|\btheta_{\MAP}) + \log\, p(\btheta_{\MAP}|\alpha)  \\
      &- \frac{1}{2} \log\det\left( \frac{1}{2\pi}\left(\textsc{ggn} + \alpha \mathbb{I}\right) \right).
\end{split}
\end{align}
This can then be optimized numerically to estimate $\alpha$. In contrast, the projective posterior \eqref{eq:approx_post} does not require optimization as $\alpha$ is available in closed form.
\begin{lemma}
    The marginal likelihood for the projected posterior \eqref{eq:approx_post} has a globally optimal $\alpha$ given by
    \begin{equation}\label{eq:opt_alpha}
        \alpha^* = \frac{\norm{\btheta_{\MAP}}^2}{P - \trace(\mathbb{I}_P - \mathcal{P}\left(\textnormal{\textsc{ggn}}_{\btheta_{\MAP}}\right))}.
    \end{equation}
\end{lemma}
In practice, the trace can be approximated using \citeauthor{hutchinson1989stochastic}'s estimator [\citeyear{hutchinson1989stochastic}].
 

\paragraph{Links to {\scriptsize LLA}.}
The projected posterior can be viewed as a fully correlated tractable approximation to the \textsc{lla}. The following statement shows that the difference between the two approximate posteriors is bounded.
\begin{lemma}\label{lm:error_bound}
   Let $\tau$ denote the smallest non-zero eigenvalue of the \textnormal{\textsc{ggn}}\citep{nguyen2021tight}, then
\begin{equation*}
    \Big\|
    \underbrace{(\alpha \mathbb{I}_P + \textnormal{\textsc{ggn}})^{-1}}_{\textnormal{\textsc{lla} cov}} 
    -
    \underbrace{\alpha^{-1}(\mathbb{I}_P - \mathcal{P}\left(\textnormal{\textsc{ggn}}_{\btheta_{\textsc{map}}}\right)}_{\textnormal{Proj cov (scaled)}}
    \Big\| \leq \frac{1}{\tau + \alpha}
\end{equation*}
Additionally, let $k$ denote the \textnormal{\textsc{ggn}} rank and $d$ the $W_2$ Wasserstein distance, then $d^2(q_{\textnormal{\textsc{lla}}}, q_{\textnormal{\textsc{proj}}}) \leq (\tau + \alpha)^{-1} k$.
%

\end{lemma}

\section{Sampling via alternating projections}\label{sec:algorithm}
We next develop a novel scalable algorithm to sample the projected posterior.
Recall that we can simulate this posterior as $\U \U\T \epsilon + \btheta_{\MAP}$, where $\epsilon \sim \N(\vec{0}, \alpha^{-1}\mathbb{I})$ and $\U\U\T$ projects to the \textsc{ggn} kernel.
In general, the orthogonal projection onto a subspace spanned by the columns of any matrix $\M \in \R^{R \times P}$, where $R < P$, is
\begin{equation}
    \label{eq:projection_basic_expression}
    \mathcal{P}\left(\M\T \M\right)
    =
    \M\T \left(\M\M\T\right)^{-1} \M \enspace\in \R^{P \times P}.
\end{equation}
The projection into the kernel of $\M\T\M$ is then given by the matrix $\mathbb{I} - \mathcal{P}\left(\M\T \M\right)$. In our case $\M = \J_{\btheta}\mat{H}_{\btheta}^{\nicefrac{1}{2}}  \in \R^{NO \times P}$ and \cref{eq:projection_basic_expression} requires inverting a $NO \!\times\! NO$ matrix, which is infeasible even for moderate datasets.

\paragraph{Intersections of kernels.}
To arrive at a feasible algorithm, we note that the kernel of $\M\T \M$ is the kernel of a sum of positive semi-definite matrices,
\begin{align}
\label{eq: sum_ker}
\textstyle
  \mathrm{ker}\left(\sum_b \M_b\T \M_b\right) = \bigcap_b \mathrm{ker}(\M_b\T \M_b), 
\end{align}
where we have decomposed $\M$ into batches,
\begin{align}
    \M =
    \begin{pmatrix}
        \M_1 \\
        \vdots \\
        \M_B
    \end{pmatrix}
    \qquad
    \M\T \M
    =
    \sum_{b=1}^B \M_b\T \M_b.
\end{align}
Hence, we can project onto the \textsc{ggn} kernel by projecting onto the \emph{intersection} of \emph{per batch}-\textsc{ggn} kernels.

\begin{SCfigure}[1.2][t]
    \centering
    \includegraphics[width=0.45\linewidth]{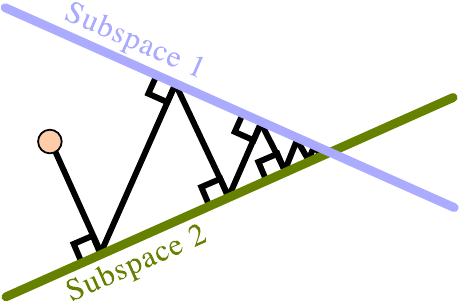}
    \caption{To project a point onto the intersection of two subspaces, we alternate between projecting onto the individual subspaces.}
    \label{fig:altproj}
\end{SCfigure}

\textbf{\mbox{\citeauthor{Neumann1949OnRO}'s [\citeyear{Neumann1949OnRO}] alternating projections}} algorithm projects onto intersections of subspaces. Fig.~\ref{fig:altproj} illustrates this idea, formalized in Lemma~\ref{lemma:alternating_projections}.
\begin{lemma}
    \label{lemma:alternating_projections}
    For any row-partition of a matrix $\M =
    \begin{pmatrix}
        \M_1\T &
        \cdots &
        \M_B\T
    \end{pmatrix}\T$ it holds that
    \begin{equation}
      \label{eq:altproj}
        \mathbb{I} - \mathcal{P}
        \Big(
        \M\T \M
        \Big) = 
        \lim_{t\rightarrow\infty}
        \Big(
          {\textstyle
            \prod_b
          }
            \left(
                \mathbb{I} - \mathcal{P}(\M_b\T \M_b)
            \right)
        \Big)^t.
    \end{equation}
    Moreover, the limit converges linearly with a rate $c\!=\!\prod_{b=1}^B \cos^2(\theta_b)$ where $\theta_b = \min_{b'\not=b} \angle(M_b,M_{b'})$ and $\angle(M_b,M_{b'})$ is the minimum angle between linear combinations of rows of $M_b$ and linear combinations of rows of $M_{b'}$. 
\end{lemma}
In practice, we stop the algorithm after a number of iterations, so that the limit \eqref{eq:altproj} stops at $t = t_{\max}$. The induced error is then upper bound by $c^{t_{\max}}$.

\paragraph{Projection for the approximate posterior.}
The algorithm projects to the $\textsc{ggn}$ kernel, by iteratively projecting to the kernels of per-batch \textsc{ggn}s. Using \cref{eq:projection_basic_expression}, a single step of the algorithm requires inverting an $SO \!\times\! SO$ matrix, which is computationally cheap for small batch-sizes $S$. Fig.~\ref{fig:projection_matrices} illustrates this pipeline.

\begin{figure*}
  \centering
  \includegraphics[width=\linewidth]{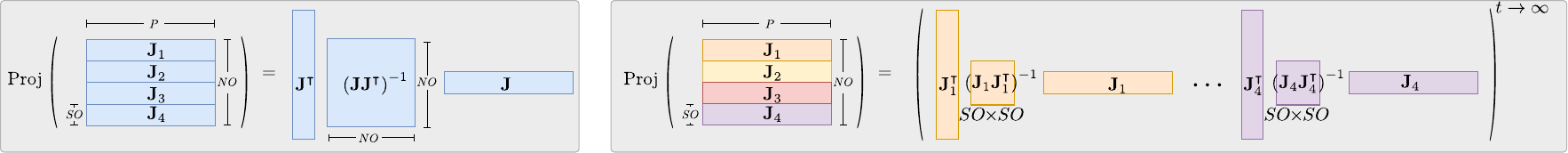}
  \vspace{-8mm}
  \caption{\textbf{Visualization of Jacobian projections} Direct calculation of the projection(left) involves inverting a large $NO\!\times\! NO$ matrix. This is replaced by an infinite (in practice, truncated) series of cheap projections (right) which only require precomputing and storing inverses of several small $SO \!\times\! SO$ matrices.}
  \label{fig:projection_matrices}
\end{figure*}

\paragraph{Implementation details.} 
For sufficiently small batches, we can precompute and store the inverses of $\M_b \M_b\T \in \R^{SO \!\times\! SO}$. Thereafter, matrix-vector products with $\M$ and $\M\T$ can be efficiently computed using Jacobian-vector products and vector-Jacobian products, respectively. We, thus, do not need to instantiate the per-batch \textsc{ggn} matrices, which gives rise to an efficient low-memory sampling algorithm.

\paragraph{Complexity.} The space complexity of the algorithm is $\mathcal{O}(P + NSO^2)$ and time complexity is $\mathcal{O}(t_{\max}PN+N S^2 O^3)$. For overparametrized models where $P \gg NO$ these complexities boil down to $\mathcal{O}(P)$ and $\mathcal{O}(t_{\max}PN)$, respectively. This matches the cost of training for $t_{\max}$ epochs.
Consequently, the projection is feasible on hardware that can train a model.

\subsection{Scaling to high-dimensional outputs} 
The computational complexity of the projection algorithm scales superlinearly with the model's output size. For many tasks this is unproblematic, but it renders the method inapplicable to e.g.\@ image-generative models.

To handle high-dimensional outputs, we propose removing the requirement of preserving the predictions of $\btheta_{\MAP}$ at the training data. Instead, we propose to (locally) preserve the \emph{loss} at each training point.
%
%
%
This is achieved by considering the loss-Jacobian, i.e.\@ the stacking of the per-datum loss-gradients. 
\begin{equation}
    \J^L_{\btheta} = \begin{pmatrix}
        \nabla_{\btheta}l(f(\btheta,\x_1),\y_1) \\
        \vdots \\
        \nabla_{\btheta}l(f(\btheta,\x_N),\y_N)
    \end{pmatrix}
    \in \R^{N \times P}
\end{equation}
\begin{lemma}
    The kernel of the stacked loss gradients $\J^L_{\btheta} \in \R^{N \times P}$ contains the kernel of the full Jacobian $\J_{\btheta} \in \R^{NO \times P}$, i.e.\@ $\mathrm{ker}(\J_{\btheta}) \subseteq \mathrm{ker}(\J^L_{\btheta})$.
    Further, note that these subspaces are identical for $O=1$.
\end{lemma}

Intuitively, for each datapoint, the gradient is a linear combination of its Jacobian rows, namely $\nabla_{\btheta}l(f(\btheta,\x),\y) =\nabla_{f(\btheta,\x)}l(f(\btheta,\x),\y)  \J_{\btheta}(\x)$. This approach can, thus, be seen as aggregating each per-datum Jacobian into a single meaningful row, lowering the row count by a factor $O$.

We propose the \emph{loss-projected approximate posterior}
\begin{equation}
    q_{\textnormal{\textsc{loss}}}(\btheta | \btheta_{\textsc{map}}, \mathcal{D})
    =
    \mathcal{N}\left(\btheta_{\textsc{map}}, \alpha^{-1}\U_L\U_L\T \right),
    \label{eq:approx_post_loss}
\end{equation} 
where $\U_L \in \R^{P \times R}$ denotes an orthogonal basis of the kernel of $\J^L_{\btheta}$. Samples from this approximate posterior are guaranteed to have the same per-datum loss as the mode parameter, up to first order,
%
\begin{restatable}{lemma}{BoundedLossKernelError}
  \label{lm:loss_proj_err}
    For any $\theta\sim q_{\textnormal{\textsc{loss}}}$ and $\x_n \in \mathcal{D}$ it holds
    \begin{align*}
        |l(f(\btheta, \x_n), \y_n) 
        \!-\!
        l(f(\btheta_{\MAP}, \x_n), \y_n)|
        = \mathcal{O}(\lVert \btheta \!-\! \btheta_{\MAP} \rVert^2)
    \end{align*}
    which implies that
    \begin{equation}
        \textnormal{\textsc{Var}}_{\btheta \sim q_{\textnormal{\textsc{loss}}}} l(f(\btheta, \x_n), \y_n)
        \leq
        \mathcal{O}(\alpha^2) .
    \end{equation}
\end{restatable}
In particular, these results hold regardless of whether we use the linearized model $f_{\textnormal{lin}}^{\btheta_{\MAP}}$ or the standard one $f$.\looseness=-1

\paragraph{Complexity.} The space complexity of the algorithm is $\mathcal{O}(P + NS)$ and the time complexity is $\mathcal{O}(t_{\max}PN + N S^2)$.

\section{Related work}\label{sec:related_work}
Deep neural networks are known to suffer from poor calibration, where predictive uncertainties are often indistinguishable between in-distribution and out-of-distribution data \citep{guo2017calibration}. \emph{Deep ensembles} \citep{Lakshminarayanan2016deepensembles} remains one of the most widely used solutions as it has strong empirical performance. Unfortunately, this approach requires retraining multiple models from scratch, making it computationally prohibitive for large models.

\textbf{Gaussian approximate posteriors} are commonly used to avoid retraining from scratch and instead only explore the mode of a single model. Established approaches in this area include \emph{Bayes by backprop} \citep{blundell2015weight}), \emph{stochastic weight averaging (\textsc{swag})} \citep{maddox2019simple}, and \emph{Laplace approximations} \citep{daxberger2021laplace}. Our work can be viewed as a variant of the Laplace approximation that avoids the underfitting commonly seen in Gaussian posteriors.

\textbf{The computational costs} of working with high-dimensional Gaussians can be daunting. Common workarounds include using low-rank or sparse covariances \citep{maddox2019simple, ritter2018scalable, deng2022accelerated}, subspace inference approaches \citep{izmailov2020subspace}, and subnetwork inference \citep{daxberger2021bayesian}. These methods reduce computational costs while preserving strong predictive performance. In contrast, our method additionally provides a rigorous theoretical justification for the choice of subspace. We can infer properties of distributions within this subspace and derive formal error bounds, a level of theoretical grounding often lacking in other approaches.

\citet{antoran2023samplingbased} use a sample-then-optimize procedure to simulate the \textsc{lla} posterior, which has some algorithmic similarities to our work. We both simulate a simple distribution followed by a gradient-based iterative refinement. While \citet{antoran2023samplingbased} focus on \textsc{lla}, we aim to avoid underfitting, which are potentially complementary objectives. Here we emphasize that our approach is general and can be applied to any Bayesian approximation in a post-hoc manner.


\textbf{A frequentist perspective} was given by \citet{Madras2020Detecting}. Their \emph{local ensembles} perturb a \textsc{map} estimate in the kernel of the network's Hessian. Algorithmically, they approximate the largest Hessian eigenvectors using Lanczos's decomposition \citep{meurant2006lanczos} and assume that the orthogonal complement is the Hessian kernel. We improve this in several ways. First, we show that the \textsc{ggn} kernel is more appropriate than the Hessian kernel. Secondly, we propose a significantly more efficient algorithm that projects onto the \textsc{ggn} kernel by leveraging its partitioned structure. Our method guarantees exact convergence to the kernel, and it is not memory-limited as it avoids storing Lanczos vectors.

\textbf{Alternating projections} first appeared in \citeauthor{Neumann1949OnRO}'s notes on operator theory [\citeyear{Neumann1949OnRO}], while \citet{kayalar1988error} proved the convergence rate. To our knowledge, the only recent machine learning use-case is inverting Gram matrices in Gaussian process regression \citep{wu2024large}. This is rather different from our work.

\section{Experiments}\label{sec:results}
We benchmark our \emph{projected} and \emph{loss-projected} posteriors against popular post-hoc Bayesian methods, including \textsc{swag}, and diagonal and last-layer Laplace. We also include the prediction \textsc{map}.
Baseline prior precisions are tuned via grid search, while projection posteriors use the optimal prior precision \eqref{eq:opt_alpha}. Details on models and hyperparameters are in Appendix C.

\subsection{Toy regression}



\begin{figure}[t]
  \centering
  \includegraphics[width=\linewidth]{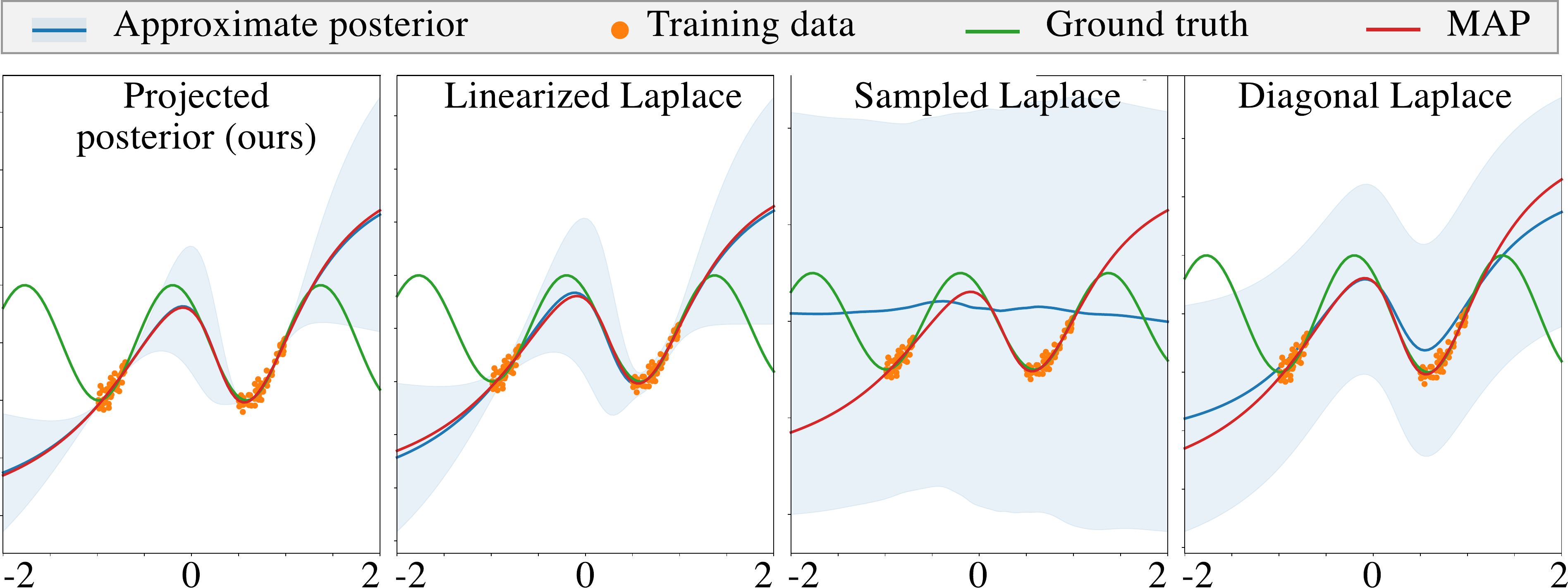}
  \vspace{-8mm}
  \caption{The sparse approximations of the posterior fail to capture in-between uncertainty whereas fully correlated posteriors can capture them.}
  \label{fig:inbetween}
\end{figure}


\begin{mdframed}[backgroundcolor=myblue, linecolor=blue!75!black]
\textbf{Hypothesis:} The projected posterior captures in-between uncertainties better than sparsely correlated Laplace.
\end{mdframed}
We first assess different Laplace approximations in a toy regression task. \citet{foong2019between} shows that Bayesian neural networks with mean-field variational inference fail to capture in-between uncertainties. 
We observe a similar issue with sparsely correlated Laplace approximations (Fig.~\ref{fig:inbetween}). In contrast, the projected posterior correlates all parameters, improving predictive uncertainty estimates for unseen regions. \emph{This illustrates that maintaining full parameter correlations is essential for accurate in-between uncertainties.}

\subsection{Predictive uncertainties on standard image classification}


\begin{mdframed}[backgroundcolor=myblue, linecolor=blue!75!black]
\textbf{Hypothesis:} The projected posterior provide competitive or superior predictive uncertainties across tasks, including in-distribution calibration, robustness to distribution shifts, and out-of-distribution detection.
\end{mdframed}

\paragraph{In-distribution performance (Table~\ref{tab:id}).}
We next consider standard image classification tasks. We train a LeNet \citep{lecun1989backpropagation} on MNIST and Fashion MNIST, and a ResNet \citep{he2016deep} on CIFAR-10 \citep{krizhevsky2009learning} and SVHN \citep{netzer2011reading}. Keeping the \textsc{map} fixed we assess the predictive posterior of all methods using metrics such as confidence, accuracy, negative log-likelihood (NLL), Brier score \citep{VERIFICATIONOFFORECASTSEXPRESSEDINTERMSOFPROBABILITY}, expected calibration error (ECE), and maximum calibration error (MCE) \citep{naeini2015obtaining}.\looseness=-1

As shown in Table~\ref{tab:id}, the projected posterior matches the \textsc{map} in-distribution, as expected from Lemma~\ref{lemma:3.1}, with zero variance in-distribution and higher variance out-of-distribution. The loss-projected posterior improves calibration, which is discussed in Appendix B.




\begin{table*}
  \caption{In-distribution performance across methods trained on \textsc{mnist}, \textsc{fmnist} \smaller[0.80]{SVHN} and \smaller[0.80]{CIFAR-10}.}
  \label{tab:id}
  %
  \setlength{\aboverulesep}{0pt}
  \setlength{\belowrulesep}{0pt}
  \setlength{\extrarowheight}{.75ex}
  \resizebox{\textwidth}{!}{
    \rotatebox[origin=c]{90}{MNIST~~~~~}\hspace{3mm}
    \begin{tabular}{lcccccc}
    \toprule \rowcolor{gray!30}
                            & Conf.~($\uparrow$)           & NLL~($\downarrow$)          & Acc.~($\uparrow$)           & Brier~($\downarrow$)        & ECE~($\downarrow$)        & MCE~($\downarrow$) \\
    \midrule
    MAP &
            0.981\small{±0.002} &
            0.080\small{±0.005} &
            \small{\first{0.977\small{±0.002}}} &
            1.775\small{±0.004} &
            0.787\small{±0.009} &
            0.895\small{±0.014} \\
    Projected posterior (ours)  &
            0.981\small{±0.002} &
            0.080\small{±0.005} &
            \small{\first{0.977\small{±0.002}}} &
            1.774\small{±0.004} &
            0.787\small{±0.009} &
            0.895\small{±0.014} \\
    Loss-projected posterior (ours) &
            0.813\small{±0.018} &
            1.225\small{±0.099} &
            0.949\small{±0.000} &
            \small{\first{1.532\small{±0.028}}} &
            \small{\first{0.666\small{±0.007}}} &
            0.894\small{±0.011} \\
    SWAG &
            \small{\first{0.982\small{±0.001}}} &
            \small{\first{0.064\small{±0.006}}} &
            0.982\small{±0.000} &
            1.776\small{±0.002} &
            0.788\small{±0.005} &
            0.906\small{±0.013} \\
    Last-Layer & 0.977\small{±0.002} & 0.090\small{±0.005} &
            0.975\small{±0.002} &
            1.768\small{±0.003} &
            0.784\small{±0.007} &
            \small{\first{0.887\small{±0.008}}} \\
    Diagonal &
            0.951\small{±0.008} &
            0.129\small{±0.016} &
            0.975\small{±0.001} &
            1.727\small{±0.012} &
            0.784\small{±0.008} &
            0.894\small{±0.015} \\

    \bottomrule
  \end{tabular} }
    %
  \resizebox{\textwidth}{!}{
    \rotatebox[origin=c]{90}{FMNIST}\hspace{3mm}
    \begin{tabular}{lcccccc}
    MAP &
        \small{\first{0.938\small{±0.005}}} &
        0.325\small{±0.011} &
        0.897\small{±0.002} &
        1.713\small{±0.008} &
        0.732\small{±0.006} &
        0.901\small{±0.006} \\
    Projected posterior (ours) &
        0.937\small{±0.005} &
        \small{\first{0.325\small{±0.011}}} &
        0.897\small{±0.002} &
        1.713\small{±0.008} &
        0.732\small{±0.006} &
        0.902\small{±0.006} \\
    Loss-projected posterior (ours) &
        0.744\small{±0.031} &
        1.529\small{±0.371} &
        0.871\small{±0.006} &
        \small{\first{1.441\small{±0.041}}} &
        \small{\first{0.617\small{±0.025}}} &
        0.901\small{±0.013} \\
    SWAG &
        0.931\small{±0.006} &
        0.327\small{±0.001} &
        \small{\first{0.898\small{±0.001}}} &
        1.703\small{±0.008} &
        0.725\small{±0.003} &
        0.907\small{±0.003} \\
    Last Layer &
        0.931\small{±0.005} &
        0.339\small{±0.011} &
        0.896\small{±0.002} &
        1.703\small{±0.008} &
        0.727\small{±0.004} &
        0.902\small{±0.004} \\

    Diagonal &
        0.735\small{±0.024} &
        0.922\small{±0.051} &
        0.855\small{±0.000} &
        1.431\small{±0.033} &
        0.621\small{±0.021} &
        \small{\first{0.895\small{±0.021}}} \\

    \bottomrule
  \end{tabular} }
   %
  \resizebox{\textwidth}{!}{
    \rotatebox[origin=c]{90}{SVHN}\hspace{3mm}
    \begin{tabular}{lcccccc}
    MAP &
            \small{\first{0.949\small{±0.004}}} &
            0.188\small{±0.004} &
            \small{\first{0.949\small{±0.002}}} &
            1.691\small{±0.004} &
            0.740\small{±0.012} &
            0.889\small{±0.004} \\
    Projected posterior (ours) &
            \small{\first{0.949\small{±0.004}}} &
            0.188\small{±0.004} &
            \small{\first{0.949\small{±0.002}}} &
            1.691\small{±0.004} &
            0.740\small{±0.012} &
            0.889\small{±0.007} \\
    Loss-projected posterior (ours) &
            0.948\small{±0.005} &
            0.191\small{±0.007} &
            \small{\first{0.949\small{±0.003}}} &
            1.685\small{±0.013} &
            \small{\first{0.734\small{±0.017}}} &
            0.880\small{±0.012} \\
    SWAG & 0.897±0.007 & 0.217±0.014 & 0.947±0.004 & \small{\first{1.606±0.010}} & 0.745±0.007 & \small{\first{0.874±0.003}} \\
    Last-Layer &
            0.943\small{±0.005} &
            0.197\small{±0.009} &
            0.946\small{±0.001} &
            1.681\small{±0.006} &
            0.740\small{±0.007} &
            0.899\small{±0.009} \\
    Diagonal &
            0.948\small{±0.003} &
            \small{\first{0.187\small{±0.005}}} &
            \small{\first{0.949\small{±0.002}}} &
            1.689\small{±0.003} &
            0.742\small{±0.011} &
            0.899\small{±0.000} \\

    \bottomrule
  \end{tabular} }
    %
  \resizebox{\textwidth}{!}{
    \rotatebox[origin=c]{90}{CIFAR-10}\hspace{3mm}
    \begin{tabular}{lcccccc}
    MAP &
            \small{\first{0.952\small{±0.002}}} &
            0.422\small{±0{.011}} &
            \small{\first{0.894\small{±0.002}}} &
            1.698\small{±0.038} &
            0.703\small{±0.013} &
            0.878\small{±0.009} \\
    Projected posterior (ours) &
            \small{\first{0.952\small{±0.001}}} &
            0.422\small{±0{.011}} &
            \small{\first{0.894\small{±0.002}}} &
            1.698\small{±0.038} &
            0.703\small{±0.013} &
            0.878\small{±0.002} \\
    Loss-projected posterior (ours) &
            0.701\small{±0.013} &
            2.643\small{±0.205} &
            0.855\small{±0.002} &
            \small{\first{1.387\small{±0.018}}} &
            \small{\first{0.559\small{±0.006}}} &
            \small{\first{0.802\small{±0.005}}} \\

    SWAG &
            0.914\small{±0.035} &
            0.445\small{±0.063} &
            0.865\small{±0.029} &
            1.670\small{±0.049} &
            0.694\small{±0.018} &
            0.881\small{±0.005} \\

    Last-Layer &
            0.944\small{±0.001} &
            \small{\first{0.406\small{±0.005}}} &
            \small{\first{0.894\small{±0.001}}} &
            1.712\small{±0.001} &
            0.704\small{±0.000} &
            0.880\small{±0.007} \\
    Diagonal &
            0.934\small{±0.028} &
            0.465\small{±0.069} &
            0.872\small{±0.032} &
            1.698\small{±0.038} &
            0.703\small{±0.013} &
            0.873\small{±0.005} \\
    \bottomrule
  \end{tabular} }
\end{table*}

\begin{table*}
  \caption{Out-of-distribution \textsc{auroc}~($\uparrow$) performance for \textsc{mnist}, \textsc{fmnist}, \smaller[0.80]{SVHN} and \smaller[0.80]{CIFAR-10}.}
  \label{tab:ood}
  %
  \setlength{\aboverulesep}{0pt}
  \setlength{\belowrulesep}{0pt}
  \setlength{\extrarowheight}{.75ex}
  %
  \resizebox{\textwidth}{!}{
    \begin{tabular}{lcccccccccc}
    \toprule 
    \rowcolor{gray!30} Trained on & \multicolumn{3}{c}{\rule[0.9mm]{20mm}{0.1mm} \textsc{mnist} \rule[0.9mm]{20mm}{0.1mm}} & \multicolumn{3}{c}{\rule[0.9mm]{20mm}{0.1mm} \textsc{fmnist} \rule[0.9mm]{20mm}{0.1mm}} & \multicolumn{2}{c}{\rule[0.9mm]{12mm}{0.1mm} \smaller[0.80]{CIFAR-10} \rule[0.9mm]{12mm}{0.1mm}} & 
    \multicolumn{2}{c}{\rule[0.9mm]{12mm}{0.1mm} \smaller[0.80]{SVHN} \rule[0.9mm]{12mm}{0.1mm}}
    \\
    \rowcolor{gray!30} Tested on & \textsc{fmnist} & \textsc{emnist} & \textsc{kmnist} & \textsc{mnist} & \textsc{emnist} & \textsc{kmnist} & \textsc{svhn} & \smaller[0.80]{CIFAR-100} & \smaller[0.80]{CIFAR-10} & \smaller[0.80]{CIFAR-100} \\ 
    \midrule
    MAP & 0.924±0.015 & 0.888±0.010 & 0.846±0.019 & 0.714±0.001 & 0.788±0.016 & 0.664±0.012 & 0.874±0.002 & 0.816±0.026 & 
    0.960±0.007 & 0.954±0.008\\
    
    Proj.\@ (ours) & \first{0.926±0.013} & 0.904±0.004 & \first{0.862±0.011} & 0.861±0.029 & 0.844±0.032 & 0.867±0.012 & \first{0.881±0.002} & \first{0.836±0.004} & 
    0.960±0.002 & 0.957±0.001\\
    
    Loss-proj.\@ (ours) & 0.899±0.011 & 0.893±0.006 & 0.856±0.002 & \first{0.914±0.035} & \first{0.928±0.007} & \first{0.907±0.026} & 0.863±0.008 & 0.800±0.013 &  
    \first{0.966±0.009} & \first{0.960±0.006}\\
    SWAG & 0.917±0.024 & \first{0.916±0.010}  & 0.861±0.032 &
            0.685±0.014 & 0.751±0.032 & 0.655±0.015 & 0.798±0.040 & 0.782±0.042 & 
            0.777±0.029 & 0.787±0.027\\
    
    Last-Layer & 0.793±0.215 & 0.782±0.182 & 0.759±0.166 &
                   0.768±0.001 & 0.824±0.016 & 0.699±0.020 & 0.811±0.024 & 0.801±0.026 &  
                   0.914±0.005 & 0.908±0.005\\
    
    Diagonal & 0.772±0.218 & 0.768±0.191 & 0.745±0.174 & 
                0.726±0.028 & 0.725±0.034 & 0.683±0.013 & 0.836±0.022 & 0.806±0.027 &  
                0.917±0.005 & 0.916±0.004\\

    \bottomrule
  \end{tabular} }
\end{table*}

\begin{SCtable*} 
  \caption{In- and out-of-distribution metrics for SWIN transformer trained on ImageNet and tested on \textsc{places365}.}
  \label{tab:imagenet}
  %
  \setlength{\aboverulesep}{0pt}
  \setlength{\belowrulesep}{0pt}
  \setlength{\extrarowheight}{.75ex}
  \resizebox{0.7\textwidth}{!}{
    \begin{tabular}{lccccccc}
    \toprule 
    \rowcolor{gray!30}
    & Conf.~($\uparrow$)           & NLL~($\downarrow$)      & Acc.~($\uparrow$)           & Brier~($\downarrow$) 
    & ECE~($\downarrow$)        & MCE~($\downarrow$) 
    & AUROC ~($\uparrow$)             \\
    \midrule
    \textsc{map} & \first{0.626} & \first{1.358} & \first{0.726} & 0.407 & 0.136 & 0.583 & 0.751 \\
Loss-projected posterior & 0.618 & 1.427 & 0.716 & \first{0.394} & \first{0.107} & \first{0.259} & \first{0.756} \\
    \bottomrule
  \end{tabular} }
\end{SCtable*}

\begin{figure*}
  \includegraphics[width=0.49\linewidth]{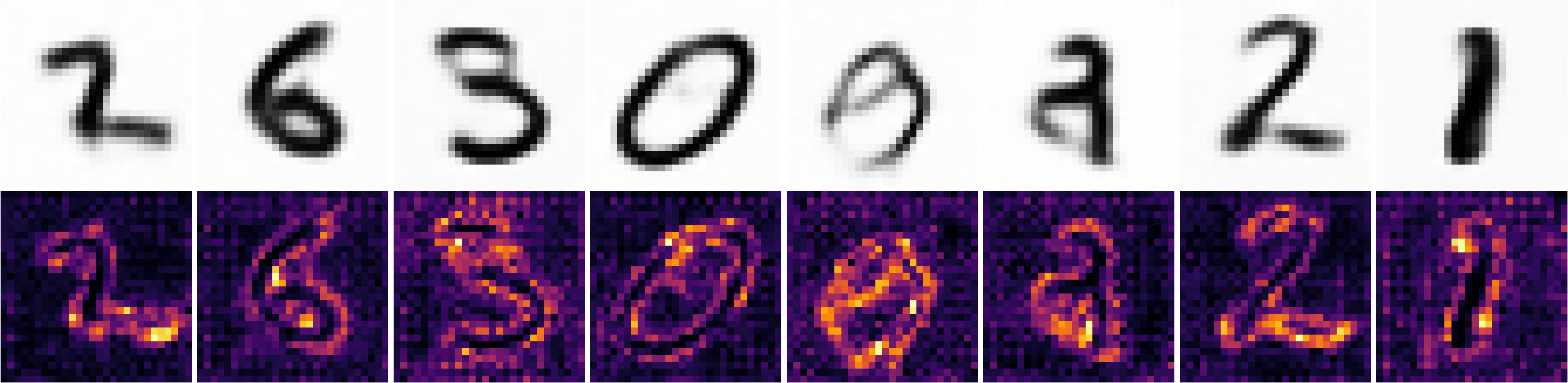}
  \includegraphics[width=0.49\linewidth]{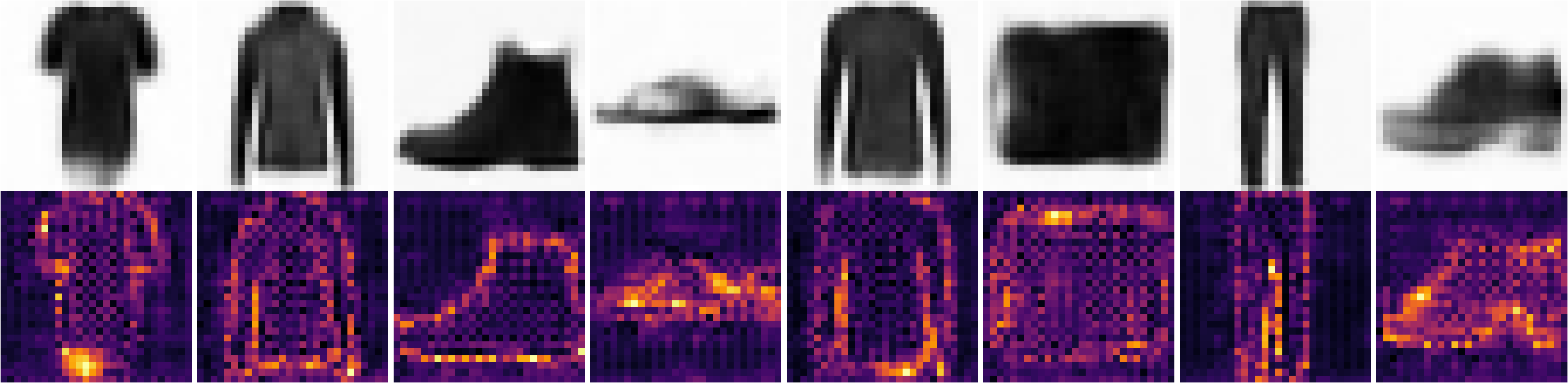}
  \vspace{-4mm}
  \caption{\textbf{Variational autoencoder reconstructions and corresponding uncertainty estimates.} \emph{Left:} mean reconstructions of MNIST and Fashion MNIST images sampled from the latent space. \emph{Right:} pixel-wise uncertainty estimates generated by sampling decoder parameters from the Loss Kernel Posterior, highlighting key semantic features such as edges and contours. This demonstrates Projected Laplace's ability to capture uncertainty in high-dimensional generative models.
}
  \label{fig:VAE variance}
\end{figure*}

\paragraph{Distribution shift (Fig.~\ref{fig:shift_plots}).}
%
%
To assess robustness under distribution shifts, we evaluate the robustness of predictive uncertainties using \textsc{rotated-mnist}, \textsc{rotated-fmnist}, and \textsc{corrupted cifar} (with fog and Gaussian blur). The loss-projected posterior maintains low, stable calibration error with increasing shift intensity, while the projected posterior retains high accuracy (Fig.~\ref{fig:shift_plots}).

\begin{figure}[t]
  \includegraphics[width=\linewidth]{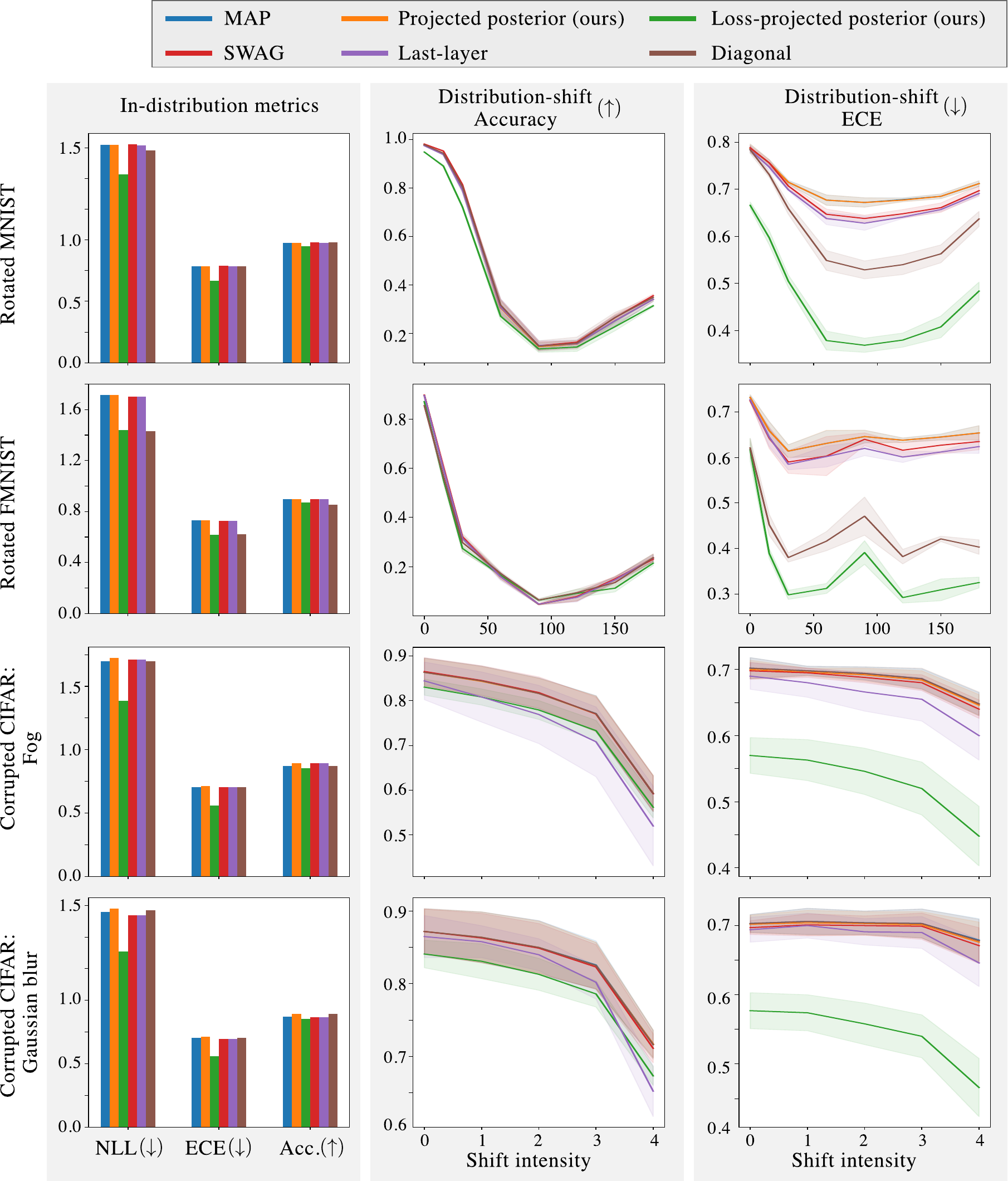}
    \vspace{-8mm}
    \caption{Model calibration and fit on in-distribution test data (left) and under distribution shift (middle, right) where we plot shift intensities against accuracy and expected calibration error (\textsc{ece}), respectively.} 
    
    \label{fig:shift_plots}
\end{figure}

\paragraph{Out-of-distribution detection (Table~\ref{tab:ood}).}
%
For out-of-distribution (\textsc{ood}) detection, we use the maximum variance of logits across output dimensions as an \textsc{ood} score for all Bayesian posteriors, while the \textsc{map} baseline employs maximum softmax, which is a strong baseline. Lemma~\ref{lemma:3.1} dictates that the in-distribution variance of the projected posterior is zero, providing a strong \textsc{ood} detection signal. Table~\ref{tab:ood} supports the theory, showing that the projected posterior achieves superior \textsc{auroc} scores across various \textsc{ood} tasks over baselines. 

\subsection{ImageNet and CelebA vision transformer}
\begin{mdframed}[backgroundcolor=myblue, linecolor=blue!75!black]
\textbf{Hypothesis:} The projected posterior scales effectively to large models and datasets, such as SWIN transformers trained on ImageNet and CelebA.
\end{mdframed}


We assess the scalability of our approach by sampling from the loss-projected posterior of a SWIN transformer \citep{liu2021swin} pre-trained on the ImageNet-1K dataset \citep{deng2009imagenet, russakovsky2015imagenet}, which includes about 1 million images of size $224\times224\times3$ with 1000 output classes. Table~\ref{tab:imagenet} shows marginal improvement in-distribution calibration over \textsc{map} alongside improved out-of-distribution detection on \textsc{places365} \citep{zhou2017places}. \emph{The main point is that our method scales gracefully to large models and datasets due to linear complexity.}

We further consider a 4M parameter Visual Attention network (\textsc{VAN})\citep{guo2023visual}. We train on \textsc{CelebA} \citep{liu2015faceattributes} holding out three classes. \cref{tab:results_celeba} reports \textsc{auroc} scores to measure the ability to detect the three held-out classes and \textsc{Food101} \citep{bossard2014food} as out-of-distribution. Both our method and baselines have a limited memory budget of $300P$  numbers. \textsc{lla} and local ensemble are implemented using Lanczos and we also include our projected score computed with Lanczos as a reference. 
The projected posterior outperforms the baselines on the three most challenging settings and is only second to \textsc{scod} in the last. The loss-projected posterior shows a minor drop in performance.

\begin{table}[t]
    \vspace{-3mm}
    \caption{\textsc{auroc} (3 seeds) for a \textsc{VAN} trained on \textsc{CelebA}. Details are in the appendix.}
    \label{tab:results_celeba}
    \setlength{\aboverulesep}{0pt}
    \setlength{\belowrulesep}{0pt}
    \setlength{\extrarowheight}{.75ex}
    \resizebox{\linewidth}{!}{
    \begin{tabular}{lcccc}
    \toprule 
    \rowcolor{gray!30}
        & Food101 & Bald only & Eyeglasses only & Mustache only \\ \midrule
    Max logit & 0.959$\pm$0.008  & 0.32$\pm$0.00  & 0.57$\pm$0.01  & 0.40$\pm$0.03  \\
    Deep ensemble & 0.957  & 0.71  & 0.74  & 0.61  \\
    Proj.\@ (ours) & 0.958$\pm$0.004  & \first{0.81$\pm$0.03}  & \first{0.75$\pm$0.01}  & \first{0.66$\pm$0.01}  \\
    Loss-proj.\@ (ours) & 0.947$\pm$0.007  & 0.76$\pm$0.03  & 0.72$\pm$0.01  & 0.62$\pm$0.01  \\
    Proj-Lanczos & 0.954$\pm$0.004  & 0.80$\pm$0.03  & 0.74$\pm$0.01  & 0.64$\pm$0.01  \\
    Local ensemble & 0.951$\pm$0.003  & 0.79$\pm$0.04  & 0.73$\pm$0.01  & 0.63$\pm$0.01  \\
    LLA & 0.954$\pm$0.00  & 0.80$\pm$0.03  & 0.74$\pm$0.01  & 0.64$\pm$0.01  \\
    LLA-diag & 0.797$\pm$0.021  & 0.54$\pm$0.03  & 0.57$\pm$0.02  & 0.45$\pm$0.03  \\
    SCOD & \first{0.964$\pm$0.000}  & 0.79$\pm$0.03  & 0.73$\pm$0.01  & 0.64$\pm$0.01  \\
    SWAG & 0.740$\pm$0.021  & 0.70$\pm$0.08  & 0.52$\pm$0.04  & 0.49$\pm$0.04 \\
    SLU & 0.953$\pm$0.003  & 0.80$\pm$0.03  & 0.74$\pm$0.01  & 0.64$\pm$0.01 \\
    \bottomrule
    
    \end{tabular}
    }
\end{table}

\subsection{Generative models}




\begin{mdframed}[backgroundcolor=myblue, linecolor=blue!75!black]
\textbf{Hypothesis:} The loss-projected posterior scales to generative models with large output dimensions.
\end{mdframed}

We evaluate the loss-projected posterior on variational autoencoders \citep{kingma2013auto} trained on \textsc{MNIST} and \textsc{Fashion MNIST}, showcasing its flexibility with high-dimensional outputs and diverse tasks. After training, we sample decoder parameters to generate images from latent space samples. Fig.~\ref{fig:VAE variance} shows that  the sampled decoders produce diverse reconstructions with pixel-wise uncertainty estimates that capture meaningful features.\looseness=-1

\emph{This experiment highlights that our approach scales to generative models while preserving essential predictive uncertainties.} Its adaptability beyond classification makes it a robust tool for diverse tasks.

\section{Summary and limitations}

\textbf{Theoretically} we show that our proposed \emph{projected posterior} is optimal w.r.t.\@ underfitting for noiseless data. We further show that existing Laplace approximations do not meet this fundamental requirement. However, the applicability to high-noise data remains unresolved. 


\textbf{Algorithmically} we introduce a memory-efficient algorithm for computing projection-vector-products for a Jacobian kernel. This approach is based on the novel idea of representing the kernel as an intersection of kernels, which enables the use of alternating projections.


\textbf{Empirically} our methods outperform existing baselines in out-of-distribution detection and in-distribution calibration in diverse settings, ranging from toy regression problems to vision transformers on ImageNet.

\textbf{Extending the method} can be done in several ways. We merely project an isotropic Gaussian to isolate the strength of the projection itself, but we stress that the algorithm lets us project samples from any distribution. For example, it is straightforward to project samples from last-layer Laplace into the kernel. We emphasize that our method is very general as it works for any differentiable model, unlike e.g.\@ \textsc{kfac} \citep{dangel2020backpack} which requires network-specific implementations. Our open-source implementation is applicable whenever we can access Jacobian-vector and vector-Jacobian products.\looseness=-1






\subsection*{Acknowledgements}
This work was supported by a research grant (42062) from VILLUM FONDEN.
This work was partly funded by the Novo Nordisk Foundation through the Center for Basic Research in Life Science (NNF20OC0062606).
This project received funding from the European Research Council (ERC) under the European Union's Horizon Programme (grant agreement 101125003).

\bibliography{iclr2025_conference}
\bibliographystyle{abbrvnat}

\appendix
\onecolumn
\newpage

\section{Proofs}

\subsection{Proof of positive variance lemma}
\label{sec:proof_lemma_positive_variance}
In this section, we prove \Cref{lm:variance_nonzero_nontrainset}. We restate it for convenience.

Let \mbox{$\J_{\btheta} \!=\! [\J_{\btheta}(\x_1)\T \ldots \J_{\btheta}(\x_N)\T]\T$ and} $\x_{\textnormal{test}} \in \R^I$, then
\begin{equation*}
  \begin{gathered}
     \textnormal{\textsc{Var}}_{\btheta \sim q_{\textnormal{\textsc{proj}}}} f_{\textnormal{lin}}^{\btheta_{\MAP}}(\btheta, \x_{\textnormal{test}}) > 0 \\
    \textnormal{\textsc{Rank}}\!\left(\begin{smallpmatrix}
    \J_{\btheta_{\textsc{map}}} \\ \J_{\btheta_{\textsc{map}}}(\x_{\textnormal{test}})
\end{smallpmatrix} \begin{smallpmatrix}
    \J_{\btheta_{\textsc{map}}} \\  \J_{\btheta_{\textsc{map}}}(\x_{\textnormal{test}})
\end{smallpmatrix}\T\right)
>
\textnormal{\textsc{Rank}}\left(\J_{\btheta_{\textsc{map}}} \J_{\btheta_{\textsc{map}}}\T\right).
  \end{gathered}
\end{equation*}

\begin{proof}
It is easier to show the contrapositive i.e.\@ to show that if the variance is 0, then the square matrix has rank at most $NO$. 

Let's first define a more compact notation $\J_D = \J_{\btheta_{\textsc{map}}}$ and $\J_T=\J_{\btheta_{\textsc{map}}}(\x_{test})$ for the two Jacobians. Let $D=\textnormal{\textsc{Rank}}(\J_T)$ and $T=\textnormal{\textsc{Rank}}(\J_T$). It holds that $D\leq NO$ and $T\leq O$ since the rank is always upper bound by the number of rows. Consider their singular values decompositions
\begin{equation}
    J_D = \sum_{i=1}^{D} \sigma_i v_i V_i^\top
    \qquad
    J_T = \sum_{j=1}^{T} \gamma_j w_j W_j^\top,
\end{equation}
where $\sigma_i,\gamma_j>0$, $v_i\in\mathbb{R}^{D}$, $V_i\in\mathbb{R}^{P}$, $w_j\in\mathbb{R}^{T}$, $ W_j\in\mathbb{R}^{P}$, $v_i^\top v_{i'}=V_i^\top V_{i'}=\delta_{i i'}$ and $w_j^\top W_{j'}=W_j^\top W_{j'}=\delta_{j j'}$ for any $i,i'=1,\ldots,D$ and $j,j'=1,\ldots,T$. Moreover let $V_{D+1},\ldots,V_P\in\mathbb{R}^P$ be a orthonormal completion of $V_1,\ldots,V_{D}$ as a basis.
\begin{align}
    \textnormal{\textsc{Var}}_{\btheta \sim q} f_{\textnormal{lin}}^{\btheta_{\MAP}}(\btheta, \x_{test})
    & =
    \textnormal{\textsc{Tr}}
    \left(
        \sum_{j,j'=1}^{T}
        \gamma_j w_j W_j^\top
        \overbrace{
        \left(
            \sum_{i=D+1}^P
            V_i V_i^\top
        \right)}^{\mathbb{I}-\mathcal{P}(\textsc{ggn})}
        \gamma_{j'} W_{j'} w_{j'}^\top
    \right) \\
    & =
    \sum_{i=D+1}^P
    \sum_{j=1}^{T}
    \gamma_j^2 W_j^\top
    V_i V_i^\top
    W_{j} \\
    & =
    \sum_{i=D+1}^P
    \sum_{j,=1}^{T}
    \gamma_j^2
    (W_j^\top V_i)^2
\end{align}
and a sum of positive elements being 0 implies that every element is 0, thus
\begin{equation}
    \textnormal{\textsc{Var}}_{\btheta \sim q} f_{\textnormal{lin}}^{\btheta_{\MAP}}(\btheta, \x_{test}) = 0
    \quad\Longrightarrow\quad
    W_j \perp V_i 
    \quad \forall j=1,\ldots T
    \quad \forall i=D+1,\ldots,P
\end{equation}
and since $V_1,\ldots,V_P$ is a basis, this implies that there exists some coefficients $\beta^{(j)}_k$ such that $W_j = \sum_{k=1}^{D} \beta^{(j)}_k V_k$. Consequently
\begin{equation}
\label{eq:blblbl_some_random_label}
    W_j W_j^\top
    =
    \sum_{k,k'=1}^{D} \beta^{(j)}_k V_k \beta^{(j)}_{k'} V_{k'}^\top
    =
    \sum_{k=1}^{D} (\beta^{(j)}_k)^2 V_k V_{k}^\top
    \quad \forall j=1,\ldots O
\end{equation}
Now consider the matrix
\begin{align}
    \begin{pmatrix}
        \J_D \\ \J_T
    \end{pmatrix}^\top
    \begin{pmatrix}
        \J_D \\ \J_T
    \end{pmatrix} 
    & =
    \J_D^\top\J_D
    +
    \J_T^\top\J_T \\
    & =
    \sum_{i=1}^D \sigma_i V_i V_i^\top
    +
    \sum_{j=1}^T \gamma_j W_j W_j^\top \qquad\quad (\cref{eq:blblbl_some_random_label})\\
    & =
    \sum_{i=1}^D \sigma_i V_i V_i^\top
    +
    \sum_{j=1}^T \gamma_j \sum_{k=1}^{D} (\beta^{(j)}_k)^2 V_k V_{k}^\top \\
    & =
    \sum_{i=1}^D \sigma_i V_i V_i^\top
    +
    \sum_{k=1}^{D} \left( \sum_{j=1}^T \gamma_j  (\beta^{(j)}_k)^2 \right) V_k V_{k}^\top  \\
    & =
    \sum_{i=1}^D 
    \left( 
        \sigma_i 
        +
        \sum_{j=1}^T \gamma_j  (\beta^{(j)}_i)^2
    \right)    
    V_i V_i^\top
\end{align}
which has rank equal to $D$. Then finally
\begin{equation}
    \textnormal{\textsc{Rank}}
    \left(
    \begin{pmatrix}
        \J_D \\ \J_T
    \end{pmatrix}
    \begin{pmatrix}
        \J_D \\ \J_T
    \end{pmatrix} ^\top
    \right)
    =
    \textnormal{\textsc{Rank}}
    \left(
    \begin{pmatrix}
        \J_D \\ \J_T
    \end{pmatrix}^\top
    \begin{pmatrix}
        \J_D \\ \J_T
    \end{pmatrix} 
    \right)
    =
    D \leq NO
\end{equation}
which concludes the proof.
\end{proof}

\subsection{Proof of the upper bound on Linearized Laplace predictive variance}
\label{sec:proof_upper_bound_lla}
In this section, we prove \cref{thm:upper_bound_lla}. We restate it below for convenience
\UpperBoundLinearizedLaplace*

\begin{proof}
Let $\J_{\btheta_{\textsc{map}}}=\sum_{i=1}^{NO}\sigma_i w_i^\top v_i$ be a SVD decomposition of the full dataset Jacobian. Namely, $w_i\in\mathbb{R}^{NO}$, $v_i\in\mathbb{R}^P$ for any $i$, and $w_i^\top w_i=\delta_{ij}$, $v_i^\top v_i=\delta_{ij}$ for any $i, j$. Then it holds
\begin{equation}
    (\textnormal{\textsc{ggn}} + \alpha \mathbb{I})^{-1}
    =
    \alpha^{-1} \mathbb{I} +
    \sum_{i=1}^{NO} 
    \left(\frac{1}{\sigma_i^2 +\alpha} - \frac{1}{\alpha}\right) v_i^\top v_i
\end{equation}
and consequently
\begin{align}
    \J_{\btheta_{\textsc{map}}}
    (\textnormal{\textsc{ggn}} + \alpha \mathbb{I})^{-1}
    \J_{\btheta_{\textsc{map}}}^\top
    & =
    \sum_{j,k=1}^{NO}
    \sigma_j w_j^\top v_j
    \left(
        \alpha^{-1} \mathbb{I} +
        \sum_{i=1}^{NO} 
        \left(\frac{1}{\sigma_i^2 +\alpha} - \frac{1}{\alpha}\right) v_i^\top v_i
    \right)
    \sigma_k v_k^\top w_k \\
    & =
    \sum_{i=1}^{NO}
    \frac{\sigma_i^2}{\alpha}
    w_i^\top w_i
    +
    \sum_{i=1}^{NO}
    \sigma_i^2
    \left(\frac{1}{\sigma_i^2 +\alpha} - \frac{1}{\alpha}\right)
    w_i^\top w_i \\
    & =
    \sum_{i=1}^{NO}
    \frac{\sigma_i^2}{\sigma_i^2 +\alpha}
    w_i^\top w_i
\end{align}
Now note that each datapoint Jacobian can be written in terms of the full dataset Jacobian as $\J_{\btheta_{\textsc{map}}}(\x_n) = \sum_{i=1}^{O} f_i^\top e_{(n-1)O+i} \J_{\btheta_{\textsc{map}}}$ where $\{f_i\}_{i=1\ldots O}$ and $\{e_i\}_{i=1\ldots NO}$ are the canonical bases of $\mathbb{R}^O$ and $\mathbb{R}^{NO}$, respectively.

We can finally express the Linearized Laplace predictive variance, for any train datapoint $\x_n$, as 
\begin{align}
\label{eq:lla_predictive_variance_expanded}
    \textnormal{\textsc{Var}}_{\theta\sim \mathcal{N}(\btheta_{\MAP}, (\textnormal{\textsc{ggn}} + \alpha \mathbb{I})^{-1})} 
    & f_{\text{lin}}^{\btheta_{\MAP}}(\btheta, \x_n)
    =
    \textnormal{\textsc{Tr}}
    \left(
    \J_{\btheta_{\textsc{map}}}(\x_n)
    (\textnormal{\textsc{ggn}} + \alpha \mathbb{I})^{-1}
    \J_{\btheta_{\textsc{map}}}(\x_n)^\top
    \right) \\ \nonumber
    = &
    \textnormal{\textsc{Tr}}
    \left(
    \sum_{i,j=1}^{O} f_i^\top e_{(n-1)O+i} 
    \J_{\btheta_{\textsc{map}}}(\x_n)
    (\textnormal{\textsc{ggn}} + \alpha \mathbb{I})^{-1}
    \J_{\btheta_{\textsc{map}}}^\top
    e_{(n-1)O+j}^\top f_j
    \right) \\ \nonumber
    = &
    \sum_{i=1}^{O}
    e_{(n-1)O+i} \J_{\btheta_{\textsc{map}}}
    (\textnormal{\textsc{ggn}} + \alpha \mathbb{I})^{-1}
    \J_{\btheta_{\textsc{map}}}^\top
    e_{(n-1)O+i}^\top \\ \nonumber
    = & 
    \sum_{i=1}^{O}
    \sum_{j=1}^{NO}
    \frac{\sigma_j^2}{\sigma_j^2 +\alpha}
    e_{(n-1)O+i} 
    w_j^\top w_j 
    e_{(n-1)O+i}^\top 
\end{align}
Now the \emph{upper bound} follows by noting that for all $j$
\begin{equation}
    \frac{\sigma_j^2}{\sigma_j^2 +\alpha} \leq 
    \max_k 
    \frac{\sigma_k^2}{\sigma_k^2 +\alpha}
\end{equation}
and thus from \cref{eq:lla_predictive_variance_expanded}, noting that $\sum_{j=1}^{NO}w_j^\top w_j =\mathbb{I}_{NO}$ we have for any train datapoint $\x_n$
\begin{align}
    \textnormal{\textsc{Var}}_{\theta\sim \mathcal{N}(\btheta_{\MAP}, (\textnormal{\textsc{ggn}} + \alpha \mathbb{I})^{-1})} 
    f_{\text{lin}}^{\btheta_{\MAP}}(\btheta, \x_n)
    & =
    \sum_{i=1}^{O}
    \sum_{j=1}^{NO}
    \frac{\sigma_j^2}{\sigma_j^2 +\alpha}
    e_{(n-1)O+i} 
    w_j^\top w_j 
    e_{(n-1)O+i}^\top \\ \nonumber
    & \leq  
    \max_k \frac{\sigma_k^2}{\sigma_k^2 +\alpha}
    \sum_{i=1}^{O}
    \sum_{j=1}^{NO}
    e_{(n-1)O+i} 
    w_j^\top w_j 
    e_{(n-1)O+i}^\top\\ \nonumber
    & = 
    \max_k \frac{\sigma_k^2}{\sigma_k^2 +\alpha}
    \sum_{i=1}^{O}
    e_{(n-1)O+i} 
    e_{(n-1)O+i}^\top
    = 
    O \max_k \frac{\sigma_k^2}{\sigma_k^2 +\alpha}
\end{align}
While the \emph{lower bound}, similarly, follows by noting that for all $j$
\begin{equation}
    \frac{\sigma_j^2}{\sigma_j^2 +\alpha} \geq 
    \min_k 
    \frac{\sigma_k^2}{\sigma_k^2 +\alpha}
\end{equation} 
and thus from \cref{eq:lla_predictive_variance_expanded} we have for any train datapoint $\x_n$
\begin{align}
    \textnormal{\textsc{Var}}_{\theta\sim \mathcal{N}(\btheta_{\MAP}, (\textnormal{\textsc{ggn}} + \alpha \mathbb{I})^{-1})} 
    f_{\text{lin}}^{\btheta_{\MAP}}(\btheta, \x_n)
    & =
    \sum_{i=1}^{O}
    \sum_{j=1}^{NO}
    \frac{\sigma_j^2}{\sigma_j^2 +\alpha}
    e_{(n-1)O+i} 
    w_j^\top w_j 
    e_{(n-1)O+i}^\top \\ \nonumber
    & \geq  
    \min_k \frac{\sigma_k^2}{\sigma_k^2 +\alpha}
    \sum_{i=1}^{O}
    \sum_{j=1}^{NO}
    e_{(n-1)O+i} 
    w_j^\top w_j 
    e_{(n-1)O+i}^\top
    = 
    O \min_k \frac{\sigma_k^2}{\sigma_k^2 +\alpha}
\end{align}
which concludes the proof by noting that the function $\sigma\mapsto\frac{\sigma}{\sigma+\alpha}$ is monotonic for any $\alpha>0$, and that the eigenvalues of the \textsc{ggn} are a simple function of the singular values of the Jacobian: $\lambda_{max}(\textnormal{\textsc{ggn}})=\max_k \sigma_k^2$ and $\lambda_{min}^{\not=0}(\textnormal{\textsc{ggn}})=\min_k \sigma_k^2$. We emphasize that the singular values of the Jacobian affect the non-zero eigenvalues of the \textsc{ggn}, thus the minimum singular value corresponds to the minimum-non-zero eigenvalue.
\end{proof}

\subsection{Proof of Lemma 3.4}
\begin{proof}
    
It follows from equation \ref{eq: lml} the approximate log marginal likelihood is given by:
\begin{equation}
    \log \, q_{\textsc{lla}}(\mathcal{D}|\alpha) 
     = 
      \log\,p(\mathcal{D}|\btheta_{\MAP}) + \log\, p(\btheta_{\MAP}|\alpha) 
      + \frac{1}{2} \log\det\left( \frac{1}{2\pi}\left(\textsc{ggn} + \alpha \mathbb{I}\right)^{-1} \right)
\end{equation}
We can approximate $(\textsc{ggn} + \alpha \mathbb{I})^{-1}$  with $\alpha^{-1}(\mathbb{I}_P - \mathcal{P}\left(\textnormal{\textsc{ggn}}_{\btheta_{\textsc{map}}}\right))$(cf. lemma \ref{lm:error_bound}). Hence log-determinant of $(\textsc{ggn} + \alpha \mathbb{I})^{-1}$ can also be approximated by sum of log-eigenvalues $\alpha^{-1}(\mathbb{I}_P - \mathcal{P}\left(\textnormal{\textsc{ggn}}_{\btheta_{\textsc{map}}}\right))$. 
\begin{align}
    \log \, q_{\textsc{lla}}(\mathcal{D}|\alpha) 
     &= 
      \log\,p(\mathcal{D}|\btheta_{\MAP}) + \log\, p(\btheta_{\MAP}|\alpha) 
      + \frac{1}{2} \log\det\left( \frac{1}{2\pi}\left(\textsc{ggn} + \alpha \mathbb{I}\right)^{-1} \right) \\
      &\approx  - \frac{1}{2} \alpha \lVert \btheta_{\MAP} \rVert^2 + \frac{P - Tr(\mathbb{I}_P - \mathcal{P}\left(\textnormal{\textsc{ggn}}_{\btheta_{\textsc{map}}}\right))}{2} \log(\alpha)  +  C
\end{align}
where $C$ denotes all the terms that don't depend on $\alpha$. Taking the derivative wrt $\alpha$ of the above equation and setting it to zero gives us the stationary points.
\begin{align}
    \frac{d \log \, q_{\textsc{lla}}(\mathcal{D}|\alpha)}{d \alpha} &\approx -  \frac{1}{2} \lVert \btheta_{\MAP} \rVert^2 + \frac{P - Tr(\mathbb{I}_P - \mathcal{P}\left(\textnormal{\textsc{ggn}}_{\btheta_{\textsc{map}}} \right))}{2} \frac{1}{\alpha} = 0 \\
    &\implies \alpha^* = \frac{\norm{\btheta_{\MAP}}^2}{P - \trace(\mathbb{I}_P - \mathcal{P}\left(\textnormal{\textsc{ggn}}_{\btheta_{\MAP}}\right))}
\end{align}

To conclude the proof one only needs to notice that the second derivative of $\log \, q_{\textsc{lla}}(\mathcal{D}|\alpha)$ is given by $-\frac{P - Tr(\mathbb{I}_P - \mathcal{P}\left(\textnormal{\textsc{ggn}}_{\btheta_{\textsc{map}}} \right))}{2} \frac{1}{\alpha^2}$ which is always negative. Hence the stationary point is the maximum value of the approximate log marginal likelihood. 
\end{proof}

\subsection{Proof of Error Bounds in Lemma 3.5}
\begin{proof}
    
To prove the bound on the matrix norm of the difference between the covariances of Laplace's approximation and projection posterior, notice that
\begin{equation*}
    (\alpha \mathbb{I}_P + \textnormal{\textsc{ggn}})^{-1} = \alpha^{-1}(\mathbb{I}_P - \mathcal{P}\left(\textnormal{\textsc{ggn}}_{\btheta_{\textsc{map}}}\right) + V (\Lambda + \alpha)^{-1} V^T
\end{equation*}
Where $\Lambda$ and $V$ correspond to non-zero eigenvalues and eigenvectors of \textsc{ggn} respectively. Therefore from the properties of the spectral norm, we have that:

\begin{align*}
    \Big\|
    (\alpha \mathbb{I}_P + \textnormal{\textsc{ggn}})^{-1} 
    -
    \alpha^{-1}(\mathbb{I}_P - \mathcal{P}\left(\textnormal{\textsc{ggn}}_{\btheta_{\textsc{map}}}\right)
    \Big\| &= \Big\| V (\Lambda + \alpha)^{-1} V^T \Big\| \\
    &\leq \frac{1}{\tau + \alpha}
\end{align*}

This proves the first bound. To prove the bound on Wasserstein distance note that the Wasserstein distance between two Gaussian, $\mathcal{N}(\mu_1, \Sigma_1)$ and $\mathcal{N}(\mu_2, \Sigma_2)$, is given by: 
\begin{equation*}
    d^2 =  \Big\| \mu_1 - \mu_2 \Big\|^2 + Tr(\Sigma_1 + \Sigma_2 - 2(\Sigma_2^{\frac{1}{2}} \Sigma_1 \Sigma_2^{\frac{1}{2}})^{\frac{1}{2}})
\end{equation*}
We plug in $\mu_1 = \mu_1 = \btheta_{\textsc{map}}$, $\Sigma_1 = (\alpha \mathbb{I}_P + \textnormal{\textsc{ggn}})^{-1} $ and $\Sigma_2 = \alpha^{-1}(\mathbb{I}_P - \mathcal{P}\left(\textnormal{\textsc{ggn}}_{\btheta_{\textsc{map}}}\right))$. Also notice that by the properties of projection matrices, it follows that  $\Sigma_2^{\frac{1}{2}} \Sigma_1 \Sigma_2^{\frac{1}{2}} =  \frac{1}{\alpha^2} (\mathbb{I}_P - \mathcal{P}\left(\textnormal{\textsc{ggn}}_{\btheta_{\textsc{map}}}\right))$. Hence we have that
\begin{align*}
    d^2 &= Tr\left((\alpha \mathbb{I}_P + \textnormal{\textsc{ggn}})^{-1} -  \alpha^{-1}(\mathbb{I}_P - \mathcal{P}\left(\textnormal{\textsc{ggn}}_{\btheta_{\textsc{map}}}\right)\right) \\
    &= Tr(V (\Lambda + \alpha)^{-1} V^T ) \\
    &\leq \frac{k}{\tau + \alpha}
\end{align*}
\end{proof}

\subsection{Proof of equation \ref{eq: sum_ker}}
\begin{proof}
    To prove that $\mathrm{ker}\left(\sum_b \M_b\T \M_b\right) = \bigcap_b \mathrm{ker}(\M_b\T \M_b)$, assume that $v\in \mathrm{ker}\left(\sum_b \M_b\T \M_b\right)$ then it follows that:
    \begin{align*}
        &\left(\sum_b \M_b\T \M_b\right)v = 0 \\
        &\implies v^T \left(\sum_b \M_b\T \M_b\right)v = 0 \\
        &\implies\sum_b \underbrace{v^T \M_b\T \M_b v}_{\geq 0} = 0 \\
        &\implies v^T \M_b\T \M_b v = 0 \quad \forall b
    \end{align*}
Hence, we can conclude that $v \in \bigcap_b \mathrm{ker}(\M_b\T \M_b)$. On the other hand, it is obvious that if $v \in \bigcap_b \mathrm{ker}(\M_b\T \M_b)$ we have that 
\begin{align*}
        & \M_b\T \M_b v = 0 \quad \forall b \\
        &\implies\sum_b \M_b\T \M_b v = 0 \\
        &\implies v \in \mathrm{ker}\left(\sum_b \M_b\T \M_b\right)
    \end{align*}
This proves that $\mathrm{ker}\left(\sum_b \M_b\T \M_b\right) = \bigcap_b \mathrm{ker}(\M_b\T \M_b)$.
\end{proof}
\subsection{Discussion of Lemma 4.1}
The Convergence of Alternating Projections for two subspaces first appeared in John Von Neuman's Lecture Notes on Operator Theory\citep{Neumann1949OnRO}. It was extended to the case of infinite convex sets in \citet{bregman1965finding}. Rate of convergence was analyzed in several works such as \citet{kayalar1988error}, \cite{smith1977practical}. This gives us some understanding of the approximation error incurred by truncating the interactive algorithm after $t$ steps.

\subsection{Proof of Lemma 4.2}
\begin{proof}
    Suppose $v \in \mathrm{ker}(\J_{\btheta})$. Then we have that $\J_{\btheta}v = 0$. Note that $\J^L_{\btheta} = \nabla_{f(\btheta,\x)}l(f(\btheta,\x),\y)  \J_{\btheta}(\x)$. Hence, we have that $ \J^L_{\btheta}v = \nabla_{f(\btheta,\x)}l(f(\btheta,\x),\y)  (\J_{\btheta}(\x)v) = \nabla_{f(\btheta,\x)}l(f(\btheta,\x),\y) \bm{0} = \bm{0}$. Thus we have that $v \in \mathrm{ker}(\J^L_{\btheta})$. Therefore, $\mathrm{ker}(\J_{\btheta}) \subseteq \mathrm{ker}(\J^L_{\btheta})$. 
\end{proof}

\subsection{Proof of Lemma 4.3}
In this section, we prove \Cref{lm:loss_proj_err}. We restate it for convenience.
\BoundedLossKernelError*

\begin{proof}
    Consider the first order Taylor expansion of $f$ around $\btheta_{\textsc{map}}$
    \begin{equation}
        f(\btheta,\x) = 
        f(\btheta_{\textsc{map}},\x)
        +
        \J_{\btheta_{\textsc{map}}}(\btheta-\btheta_{\textsc{map}})
        +
        \mathcal{O}(\|\btheta - \btheta_{\textsc{map}}\|^2)
    \end{equation}
    and the first order Taylor expansion of $l$ around $f(\btheta_{\textsc{map}},\x)$, which we shorten as $f_{\textsc{map}}$
    \begin{equation}
        l(f(\btheta,\x),\y)
        =
        l(f_{\textsc{map}},\y)
        +
        \nabla_{f_{\textsc{map}}}l(f_{\textsc{map}},\y) 
        (l(f(\btheta,\x),\y) - l(f_{\textsc{map}},\y))
        +
        \mathcal{O}(\|l(f(\btheta,\x),\y) - l(f_{\textsc{map}},\y) \|^2)
    \end{equation}
    And we can use these to write
    \begin{align}
    \label{eq:taylor_expanded_difference_loss_function_proof_lemma}
         l(f(\btheta,\x),\y) - l(f_{\textsc{map}},\y)
         & = 
         \nabla_{f_{\textsc{map}}}l(f_{\textsc{map}},\y) \J_{\btheta_{\textsc{map}}}(\btheta-\btheta_{\textsc{map}})
         +
         \mathcal{O}(\|\btheta - \btheta_{\textsc{map}}\|^2) \\
         & = 
         \nabla_{\btheta_{\textsc{map}}}l(f(\btheta_{\textsc{map}},\x),\y)
         (\btheta-\btheta_{\textsc{map}})
         +
         \mathcal{O}(\|\btheta - \btheta_{\textsc{map}}\|^2)
    \end{align}
    where we collected all the second order terms in $\btheta$ in the $\mathcal{O}$ term. \\
    Then, for any $(\x_n,\y_n)\in\mathcal{D}$ and for any $\btheta\sim q_{\textsc{loss}}$, by definition of the loss kernel in \cref{eq:approx_post_loss}, it holds that
    \begin{equation}
        \nabla_{\btheta_{\textsc{map}}} l(f(\btheta_{\textsc{map}},\x_n),\y_n)
        (\btheta-\btheta_{\textsc{map}})
        = 0
    \end{equation}
    which we plug back into \cref{eq:taylor_expanded_difference_loss_function_proof_lemma} and we get
    \begin{equation}
        l(f(\btheta,\x_n),\y_n) - l(f(\btheta_{\textsc{map}},\x_n),\y_n)
        =
        \mathcal{O}(\|\btheta - \btheta_{\textsc{map}}\|^2)
        \qquad
        \textnormal{ for any }
        \btheta\sim q_{\textsc{loss}}
    \end{equation}
    and the first part of the Lemma is proved. The variance bound directly follows since the norm $\|\btheta - \btheta_{\textsc{map}}\|$ is controlled by the precision scale $\alpha$ of $q_{\textsc{loss}}$.
\end{proof}

\begin{figure*}[b!]
  \centering
  \includegraphics[width=\linewidth]{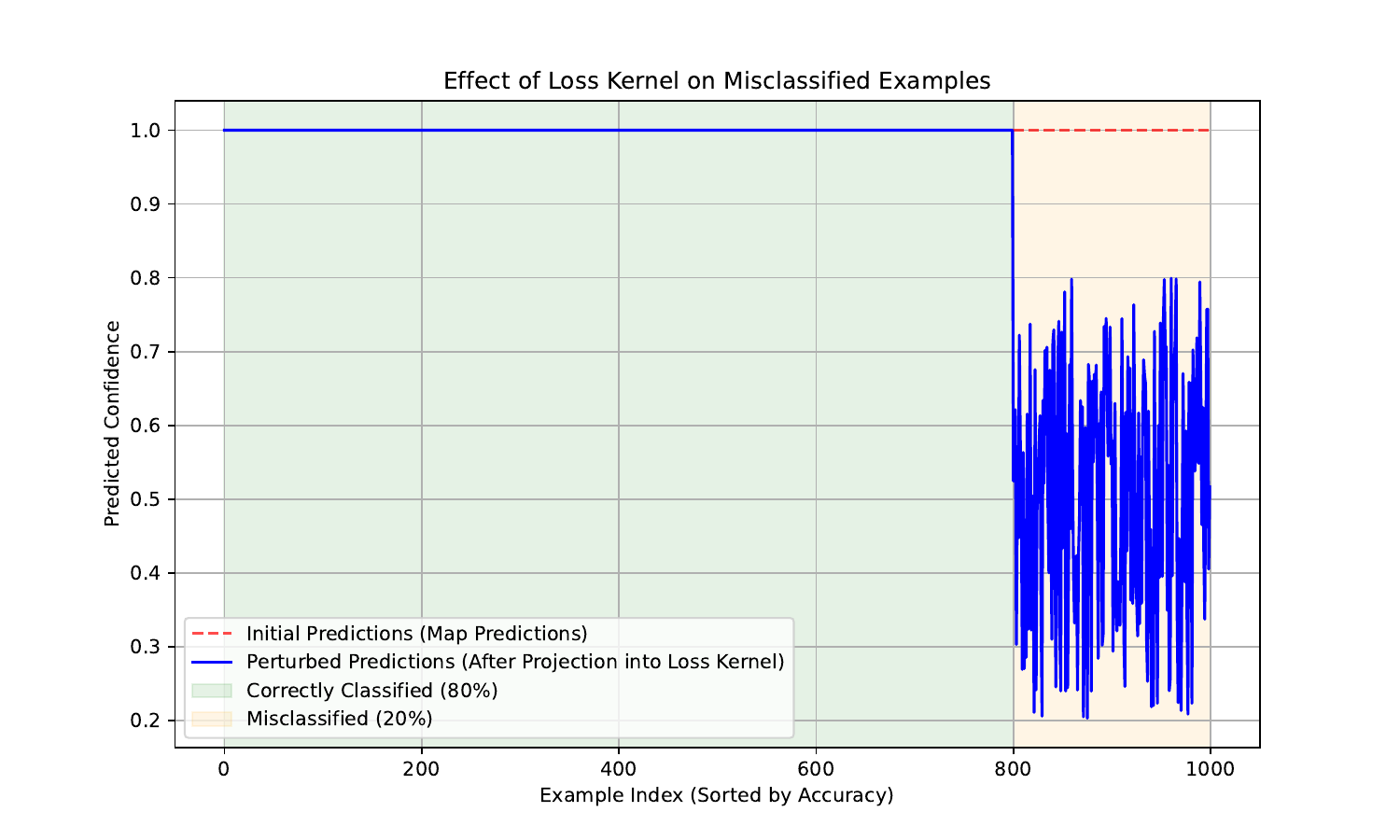}
  \vspace{-8mm}
  \caption{The loss-projected posterior perturb the predicted probabilities of all the classes except the true label. This leads to a lower confidence in misclassified examples, leading to better calibration.}
  \label{fig:calib_img}
\end{figure*}

\section{Calibration of loss kernel projection}
\label{appdx: calib}

To demonstrate how the poss-projected posterior improves calibration, consider a simple example of a classification task with $C$ output classes. Suppose the maximum a posteriori (\textsc{map}) estimate is overly confident in its predictions, consistently assigning a probability of $1.0$ to the predicted class and $0.0$ to the remaining $C - 1$ classes. For simplicity, assume this \textsc{map} estimate correctly classifies $80\%$ of the examples while misclassifying the remaining $20\%$.

In this scenario, the \textsc{map} estimate’s confidence—represented by the maximum output probability is uniformly $1.0$ for both correctly classified and misclassified examples. This results in overconfident predictions for the $20\%$ of cases where the model is wrong.

Now, consider using the loss-projected posterior for predictions. This posterior is designed to preserve the model's accuracy by maintaining the probability assigned to the true label, especially for correctly classified examples. Thus, for the $80\%$ of cases where the \textsc{map} is correct, both accuracy and confidence remain largely unchanged.

However, for the $20\%$ of misclassified examples, the loss-projected posterior adjusts the predicted probabilities. Specifically, it reduces overconfidence by perturbing the probabilities of all classes except the true label. By redistributing some of the confidence away from the incorrect class, the model achieves better calibration. Misclassified examples no longer exhibit extreme certainty, thus reducing the calibration error and yielding predictions that are more aligned with the model’s actual performance.

This behavior is illustrated in Figure~\ref{fig:calib_img}, which shows how confidence is moderated on misclassified examples, leading to an overall improvement in the calibration of the model.


\section{Implementation details and experimental setup}
\label{appdx:expt_details}

In this section, we outline the specifics of the experimental setup and provide key hyperparameter details. Any additional information about the experimental setup can be found in the submitted source code.

\subsection{Motivation for the choice of baselines}
We benchmark our method against several baselines, including \textsc{map}, Diagonal Laplace, Last-Layer Laplace, and SWAG. The motivation behind these choices is to compare our approach with other post-hoc methods that approximate a Gaussian posterior centered at the same mode. 

We chose \textsc{map} as a baseline because the maximum softmax probability is a strong baseline for out-of-distribution (OOD) detection.

The last-layer Laplace approximation was selected as a baseline following recommendations from \citep{daxberger2021laplace}. This method is considered a strong representative of Laplace approximations and is expected to provide near-optimal performance across various configurations of Laplace approximations.

We include SWAG as another baseline, which provides a simple yet effective method for uncertainty quantification. SWAG builds a Gaussian approximation close to the \textsc{map} but with a low-rank covariance matrix which is different from Laplace approximations. 

Finally, diagonal Laplace is included to assess the impact of posterior correlations. This method simplifies the full Laplace by using only diagonal covariance, providing insight into the difference between sparsely correlated and fully correlated posteriors.

We exclude KFAC-Laplace from the benchmarks because it requires specialized layer-specific implementations, which are not readily available for modern architectures like SWIN transformers.

\subsection{Toy regression}
We train a two-layer MLP with 10 hidden units per layer on a simple regression task. For uncertainty estimation, we sample from approximate posteriors of various methods, including projected posterior, loss-projected posterior, linearized Laplace, sampled Laplace, and diagonal Laplace. 

While linearized Laplace and sampled Laplace refer to the same Gaussian distribution in parameter space, they differ in how predictions are generated. Linearized Laplace uses the linearized neural network for predictions and sampled Laplace uses neural network for predictions. The diagonal Laplace method, on the other hand, approximates the covariance of this Gaussian by only considering its diagonal, which leads to sparsely correlated posteriors. For a consistent and fair comparison, all methods utilize a prior precision of $1.0$.

\subsection{Image classification on MNIST and FMNIST}

We train a standard LeNet for the MNIST and FashionMNIST experiments. We train LeNet with Adam optimizer and a $10^{-3}$ learning rate. We choose the prior precision for each baseline by doing a grid search over $\{0.1, 1.0, 5.0, 10.0, 50.0, 100.0 \}$. For Projected Posterior and Loss Projected Posteriors, we use the analytical expression for the optimal prior precision and we do $1000$ iterations of alternating projections for both and a projection batch size, $S = 16$. For all Bayesian methods, we use $30$ Monte Carlo samples for predictions. Whereas for SWAG we use a learning rate of $10^{-2}$ with momentum of $0.9$ and weight decay of $3\times10^{-4}$ and the low-rank covariance structure in all experiments. We collect $20$ models to sample from the posterior. For Projected Posterior and Loss Projected Posterior, we use the linearized predictive and for Last-Layer and Diagonal baselines we use the neural network predictive. 

\subsection{Image classification on CIFAR-10 and SVHN}
We train a ResNet architecture consisting of three groups of three ResNet blocks. The model is trained using Stochastic Gradient Descent (SGD) with a learning rate of $0.1$, momentum, and weight decay. The prior precision is chosen in the same way as the experiment above, and the same predictive functions are applied.

For the projected posteriors, we perform 1000 iterations of alternating projections with a projection batch size of $S = 16$. All Bayesian methods utilize $30$ Monte Carlo samples to compute predictions.

For SWAG on CIFAR-10, we use a learning rate of $10^{-2}$ with momentum of $0.9$ and weight decay of $3\times 10^{-4}$ and the low-rank covariance structure in all experiments. We collect 20 models to sample from the posterior.

\subsection{Image binary multi-classification on CelebA}
We first remove from the training dataset all the images belonging to the classes 'Bald', 'Eyeglasses', and 'Mustache'. Then we train a VisualAttentionNetwork \citep{guo2023visual} with blocks of depths $(3, 3, 5, 2)$ and embedded dimensions of $(32, 64, 160, 256)$ with relu activation functions, we trained with Adam for 50 epochs with batch size 128 and a learning rate decreasing from $10^{-3}$ to $10^{-5}$; parameter size is $P=3858309$. All experiments are run on a single H100 GPU.

We compare the scores of a series of baselines including Max Logit \citep{hendrycks2016baseline} and a Deep Ensemble (DE) \citep{Lakshminarayanan2016deepensembles} of 10 independently trained models, the high training cost is the reason why DE is missing standard deviations, since that would require training 30 models. We included several low-rank-approximation methods: Linearized Laplace Approximation (LLA) \citep{immer2021improving}, Local Ensemble (LE) \citep{Madras2020Detecting}, Stochastic Weight Averaging Gaussian (SWAG) \citep{maddox2019simple}, Sketching Curvature for OoD Detection (SCOD) \citep{sharma2021sketching} and Sketched Lanczos Uncertainty (SLU) \citep{miani:sketching:2024}. All of these but SLU use a rank 300 approximation for memory limit, while SLU uses a rank 1000 and a sketch size of 1M. All of these but SCOD and SWAG are based on Lanczos algorithm, thus we also computed our Projected score using Lanczos algorithm to compute the top eigenvectors, this is referred to as 'Proj-Lanczos' in \cref{tab:results_celeba}. Lastly, we also include the diagonal version of Laplace (LLA-d) which is a common method used for large models thanks to its low memory requirement.

\subsection{Image classification on ImageNet}
We obtain a pre-trained SWIN transformer\citep{liu2021swin}, with 28 million parameters, from \citep{jax-models}. We sample from the Loss Projected Posterior. We use $5$ Monte Carlo samples for predictions and do $15$ iterations of alternating projections. We use a projection batch size of $S =16$.

\subsection{Generative model}
We train a VAE with approximately $100,000$ parameters on MNIST and FMNIST. While keeping the encoder fixed we sample various decoders from the loss-projected posterior. We do $5$ iterations of alternating projections and use $20$ Monte Carlo Samples for predictions. Prior precision is chosen in the usual way.

\end{document}


%

%

\onecolumn
\aistatstitle{Appendix}

\section{Proofs}

\subsection{Proof of positive variance Lemma}
\label{sec:proof_lemma_positive_variance}
In this section, we prove \Cref{lm:variance_nonzero_nontrainset}. We restate it below for convenience.
\PositiveVarianceLemma*
\begin{proof}
It easier to show the counternominal i.e.\@ to show that if the variance is 0, then the square matrix has rank at most $NO$. 
Let's first define a more compact notation $\J_D = \J_{\btheta_{\textsc{map}}}$ and $\J_T=\J_{\btheta_{\textsc{map}}}(\x_{test})$ for the two Jacobians. Let $D=rank(\J_T)$ and $T=rank(\J_T$). It holds that $D\leq NO$ and $T\leq O$ since rank is always upper bounded by the number of rows. Consider their singular values decompositions
\begin{equation}
    J_D = \sum_{i=1}^{D} \sigma_i v_i V_i^\top
    \qquad
    J_T = \sum_{j=1}^{T} \gamma_j w_j W_j^\top
\end{equation}
where $\sigma_i,\gamma_j>0$, $v_i\in\mathbb{R}^{D}$, $V_i\in\mathbb{R}^{P}$, $w_j\in\mathbb{R}^{T}$, $ W_j\in\mathbb{R}^{P}$, $v_i^\top v_{i'}=V_i^\top V_{i'}=\delta_{i i'}$ and $w_j^\top W_{j'}=W_j^\top W_{j'}=\delta_{j j'}$ for any $i,i'=1,\ldots,D$ and $j,j'=1,\ldots,T$. Moreover let $V_{D+1},\ldots,V_P\in\mathbb{R}^P$ be a orthonormal completion of $V_1,\ldots,V_{D}$ as a basis.
\begin{align}
    \textnormal{\textsc{Var}}_{\btheta \sim q} f_{\textnormal{lin}}^{\btheta_{\MAP}}(\btheta, \x_{test})
    & =
    \textnormal{\textsc{Tr}}
    \left(
        \sum_{j,j'=1}^{T}
        \gamma_j w_j W_j^\top
        \overbrace{
        \left(
            \sum_{i=D+1}^P
            V_i V_i^\top
        \right)}^{\mathbb{I}-\mathcal{P}(\textsc{ggn})}
        \gamma_{j'} W_{j'} w_{j'}^\top
    \right) \\
    & =
    \sum_{i=D+1}^P
    \sum_{j=1}^{T}
    \gamma_j^2 W_j^\top
    V_i V_i^\top
    W_{j} \\
    & =
    \sum_{i=D+1}^P
    \sum_{j,=1}^{T}
    \gamma_j^2
    (W_j^\top V_i)^2
\end{align}
and a sum of positive elements being 0 implies that every element is 0, thus
\begin{equation}
    \textnormal{\textsc{Var}}_{\btheta \sim q} f_{\textnormal{lin}}^{\btheta_{\MAP}}(\btheta, \x_{test}) = 0
    \quad\Longrightarrow\quad
    W_j \perp V_i 
    \quad \forall j=1,\ldots T
    \quad \forall i=D+1,\ldots,P
\end{equation}
and since $V_1,\ldots,V_P$ is a basis, this implies that there exists some coefficients $\beta^{(j)}_k$ such that $W_j = \sum_{k=1}^{D} \beta^{(j)}_k V_k$. Consequently
\begin{equation}
\label{eq:blblbl_some_random_label}
    W_j W_j^\top
    =
    \sum_{k,k'=1}^{D} \beta^{(j)}_k V_k \beta^{(j)}_{k'} V_{k'}^\top
    =
    \sum_{k=1}^{D} (\beta^{(j)}_k)^2 V_k V_{k}^\top
    \quad \forall j=1,\ldots O
\end{equation}
Now consider the matrix
\begin{align}
    \begin{pmatrix}
        \J_D \\ \J_T
    \end{pmatrix}^\top
    \begin{pmatrix}
        \J_D \\ \J_T
    \end{pmatrix} 
    & =
    \J_D^\top\J_D
    +
    \J_T^\top\J_T \\
    & =
    \sum_{i=1}^D \sigma_i V_i V_i^\top
    +
    \sum_{j=1}^T \gamma_j W_j W_j^\top \qquad\quad (\cref{eq:blblbl_some_random_label})\\
    & =
    \sum_{i=1}^D \sigma_i V_i V_i^\top
    +
    \sum_{j=1}^T \gamma_j \sum_{k=1}^{D} (\beta^{(j)}_k)^2 V_k V_{k}^\top \\
    & =
    \sum_{i=1}^D \sigma_i V_i V_i^\top
    +
    \sum_{k=1}^{D} \left( \sum_{j=1}^T \gamma_j  (\beta^{(j)}_k)^2 \right) V_k V_{k}^\top  \\
    & =
    \sum_{i=1}^D 
    \left( 
        \sigma_i 
        +
        \sum_{j=1}^T \gamma_j  (\beta^{(j)}_i)^2
    \right)    
    V_i V_i^\top
\end{align}
which has $rank$ equal to $D$. Then finally
\begin{equation}
    rank
    \left(
    \begin{pmatrix}
        \J_D \\ \J_T
    \end{pmatrix}
    \begin{pmatrix}
        \J_D \\ \J_T
    \end{pmatrix} ^\top
    \right)
    =
    rank
    \left(
    \begin{pmatrix}
        \J_D \\ \J_T
    \end{pmatrix}^\top
    \begin{pmatrix}
        \J_D \\ \J_T
    \end{pmatrix} 
    \right)
    =
    D \leq NO
\end{equation}
which concludes the proof.
\end{proof}

To prepare a supplementary pdf file, we ask the authors to use \texttt{aistats2025.sty} as a style file and to follow the same formatting instructions as in the main paper.
The only difference is that the supplementary material must be in a \emph{single-column} format.
You can use \texttt{supplement.tex} in our starter pack as a starting point, or append the supplementary content to the main paper and split the final PDF into two separate files.

Note that reviewers are under no obligation to examine your supplementary material.

\section{Calibration of Loss Kernel Projection}
\label{appdx: calib}

The supplementary materials may contain detailed proofs of the results that are missing in the main paper.

\subsection{Proof of Lemma 3}

\textit{In this section, we present the detailed proof of Lemma 3 and then [ ... ]}

\section{Implementation details and experimental setup}
\label{appdx: expt_details}

If you have additional experimental results, you may include them in the supplementary materials.

\subsection{The Effect of Regularization Parameter}

\textit{Our algorithm depends on the regularization parameter $\lambda$. Figure 1 below illustrates the effect of this parameter on the performance of our algorithm. As we can see, [ ... ]}

\vfill